\documentclass[lettersize,journal]{IEEEtran}
\usepackage{amsmath,amsfonts}
\usepackage{algorithmic}
\usepackage{algorithm}
\usepackage{array}
\usepackage{caption}
\usepackage{subcaption}
\usepackage{textcomp}
\usepackage{stfloats}
\usepackage{url}
\usepackage{verbatim}
\usepackage{graphicx}
\usepackage{cite}
\usepackage[switch]{lineno}
\usepackage{amsthm}
\usepackage{soul}
\usepackage{cleveref}
\usepackage[dvipsnames]{xcolor}
\usepackage{mathtools}
\newtheorem{theorem}{Theorem}
\newtheorem{corollary}{Corollary}
\newtheorem{example}{Example}
\newtheorem{lemma}{Lemma}

\begin{document}

\title{Cooperation for Scalable Supervision\\of Autonomy in Mixed Traffic}

\author{Cameron Hickert,~\IEEEmembership{Member,~IEEE,} Sirui Li,~\IEEEmembership{Member,~IEEE,} Cathy Wu,~\IEEEmembership{Member,~IEEE}
\thanks{Cameron Hickert is with the Institute for Data, Systems, and Society, Massachusetts Institute of Technology,
        Cambridge, MA 02139, USA.
        {\tt\small chickert@mit.edu}}%
\thanks{Sirui Li is with the Institute for Data, Systems, and Society, Massachusetts Institute of Technology,
        Cambridge, MA 02139, USA.
        {\tt\small siruil@mit.edu}}%
\thanks{Cathy Wu is with the Laboratory for Information \& Decision Systems; the Institute for Data, Systems, and Society; and the Department of Civil and Environmental Engineering, Massachusetts Institute of Technology,
        Cambridge, MA 02139, USA.
        {\tt\small cathywu@mit.edu}}%
}
\markboth{IEEE TRANSACTIONS ON ROBOTICS, April~2023}{}

\IEEEpubid{10.1109/TRO.2023.3262120~\copyright~2023 IEEE}

\maketitle

\begin{abstract}
Advances in autonomy offer the potential for dramatic positive outcomes in a number of domains, yet enabling their safe deployment remains \textcolor{black}{an open problem}.
\textcolor{black}{This work's motivating question is: In safety-critical settings, can we avoid the need to have one human supervise one machine at all times? The work formalizes this \textit{scalable supervision} problem by considering remotely-located human supervisors and investigates how autonomous agents can cooperate to achieve safety.} 
 The paper focuses on the safety-critical context of autonomous vehicles (AVs) merging into traffic consisting of a mixture of AVs and human drivers.
 \textcolor{black}{The analysis establishes high reliability upper bounds on human supervision requirements. It further shows that AV cooperation can improve supervision reliability by \textit{orders of magnitude} and counterintuitively requires fewer supervisors (per AV) as more AVs are adopted.}
 \textcolor{black}{These analytical results leverage queuing-theoretic analysis, order statistics, and a conservative, reachability-based approach.}
\textcolor{black}{A key takeaway is the potential value of cooperation in enabling the deployment of autonomy at scale.}
While this work \textcolor{black}{focuses on} AVs, the scalable supervision framework \textcolor{black}{may be of independent interest} to a broader array of autonomous control challenges. 
\end{abstract}


\section{Introduction}
\IEEEPARstart{G}{iven} the complexity, chaos, and unpredictability of real-world environments, safety is a critical challenge for autonomous systems.
This is particularly the case for mixed autonomy systems, which refer to systems in which humans and machines both exhibit control in the same environment~\cite{wu2018learning}.
Indeed, in an Allianz Global Assistance survey investigating decreasing interest in self-driving cars (one of the most-discussed mixed autonomy systems today), over 70\% of respondents cited safety as a reason for their lack of interest, representing increased concern from the year prior~\cite{Allianz2018}.
Thus, improving safety in mixed autonomy systems is not only an end in itself, but also is a prerequisite for unlocking the benefits that such systems may offer.

\begin{figure} [t]
    \centering
    \includegraphics[width=0.48\textwidth]{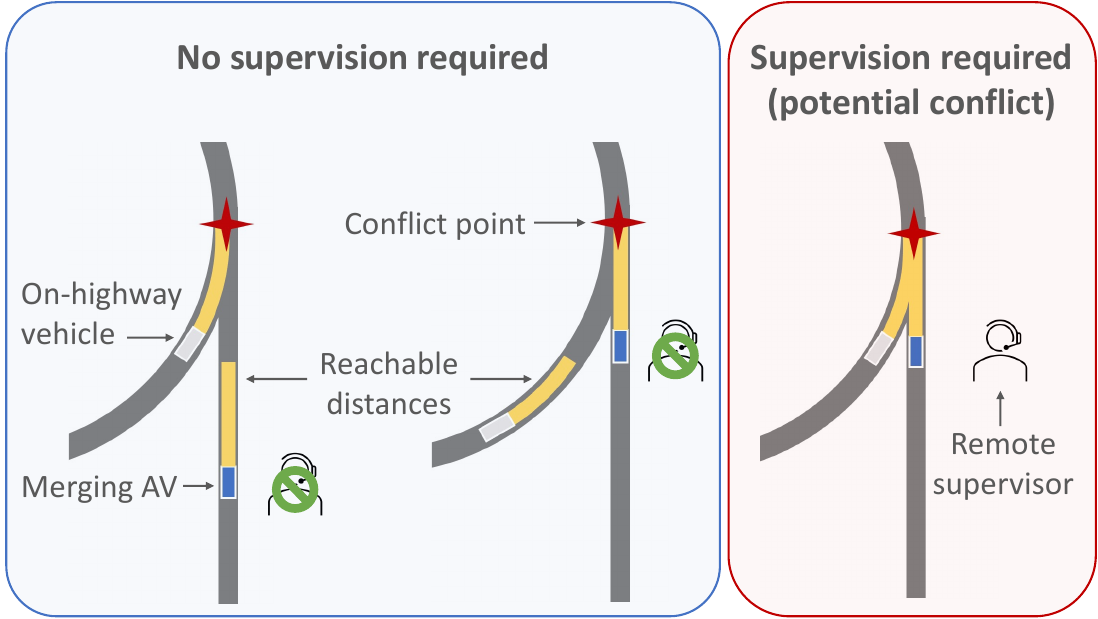}
    \caption{A general illustration of the reachability conditions that would activate a human supervisor. The yellow rectangular zones represent the reachable distances over time horizon $t$ for each vehicle. When the merge point (marked by the red cross) is within both vehicles' reachable zones, there is a possibility of an accident in the time horizon and the supervisor activates.
    }
    \label{fig:reach_conditions}
\end{figure}

The question becomes: How can we achieve \textcolor{black}{sufficient levels of safety} and performance even when formal guarantees are elusive in a chaotic world? 
\textcolor{black}{To investigate this topic, our research borrows the idea of remote supervision from air traffic control, in which a centralized controller coordinates agents for safety. }

\subsection{Contributions}
\IEEEpubidadjcol
This work provides one perspective on the topic of improving safety in mixed autonomy settings by investigating the online human supervision of AVs in an illustrative merging task.
To this end, we present a number of contributions that---to the best of our knowledge---have not been addressed in previous work.
\begin{enumerate}
    \item We formalize the \textit{scaling supervision} problem for online autonomous agents and propose a `scalability' metric: expected number of AVs per supervisor. \textcolor{black}{In the process, we propose the first (to our knowledge) large-scale remote supervision and control framework for AVs and present queuing-theoretic and statistical tools for analysis of such systems.}
    \item We propose a reachability analysis-based method for AV supervision. We provide a high \textcolor{black}{reliability} upper bound on the expected supervision time required for a supervisor to monitor a given AV's merge, as well as a closed-form expression for the probability that the number of supervisors is insufficient when merging tasks are pooled together and distributed across multiple human supervisors. 
    \item Leveraging order statistics, we show that---interestingly---cooperation of AVs enables supervision time \textit{inversely} related to AV adoption \textcolor{black}{(i.e., the proportion of vehicles in a traffic network that are AVs)\textcolor{black}{, as well as orders-of-magnitude improvement in supervision requirements. An implication of this result is that cooperation substantially facilitates safe autonomous deployments at scale.}} 
\end{enumerate}

\textcolor{black}{These outcomes suggest the adage, ``An ounce of prevention is worth a pound of cure," applies here. That is, preventing the need for supervision is more powerful than having more supervisors.} 

\section{Background}
\begin{figure} [t]
    \centering
    \includegraphics[width=0.37\textwidth]{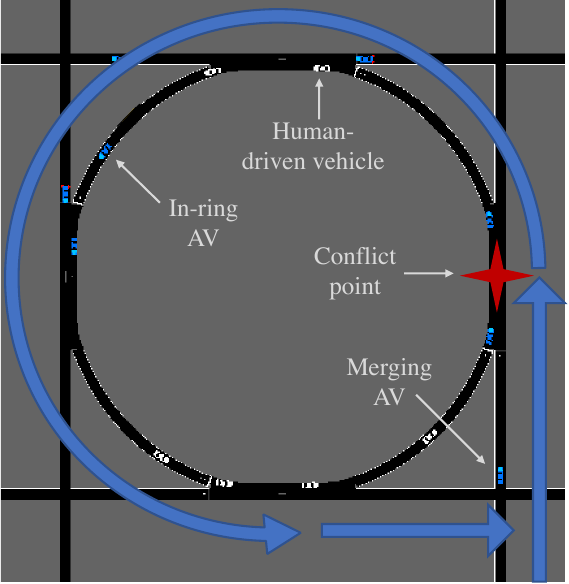}
    \caption{An illustrative vehicle merging scenario. 
    \textcolor{black}{A merging AV's path is indicated by the blue arrows; it enters via an on‐ramp on the right side of the circular highway and exits via the off‐ramp at the bottom. 
    This AV's merge point is indicated with the red marker. Human vehicles (HVs) are shown in grey.
    AVs are shown in blue.
    (Not to scale; modified for visibility.)} 
    }
    \label{fig:merge_schematic}
\end{figure}

\subsection{Motivation}
\paragraph{Safety in AVs}
The question of AV safety is an open and well-documented problem~\cite{bila2016vehicles, nascimento2019systematic}, particularly in mixed autonomy traffic settings~\cite{arbabzadeh2017data, vanholme2012highly}.
Thus, the importance of safety in any deployment of AVs is sufficient justification for investigating supervision of autonomous agents in itself.
\textcolor{black}{Human supervision of autonomous and semi-autonomous agents has been shown to improve system safety in other settings and is the de facto safety backup method for driver-assist systems on the road today~\cite{saunders2017trial, chen2014human, ziebinski2017review}.} 
If AVs replaced even a modest portion of US vehicles, the number of lives depending on AV safety systems could number in the thousands~\cite{FHA2019}.

Moreover, AVs are not the only autonomous system with the potential to inflict human harm. 
The authors anticipate that the approach and lessons learned in this research can inspire future work to improve the safety and performance of autonomous agents in settings beyond society's roadways, such as in manufacturing and other cyber-physical systems. 

\textcolor{black}{This work adopts an operational reliability perspective on safety.
That is, the proposed method does not guarantee crash-free behavior, but rather focuses on identifying the resources necessary to achieve human-level safety with an arbitrary desired level of reliability.
This is achieved via human supervision of the autonomous agents and thus the system's ultimate safety remains subject to human performance.
The concept of remote supervision is borrowed from air traffic control, in which a human supervisor centrally coordinates agents for safety~\cite{hopkin1988air}.}

\paragraph{AVs' Benefits to Traffic Systems}
While ensuring a sufficient level of safety is an important end in itself, doing so is foundational to the wider adoption of AVs, which can in turn bring benefits including decreased traffic congestion and traffic emissions. 
AVs can achieve outsized positive impact without dominating roadways. 
Using AV policies learned via reinforcement learning (RL), previous work shows that even when AVs account for only 5-10\% of vehicles in a traffic network, they can still boost the collective average speed of \textit{all} vehicles in the network by up to 57\% in idealized settings~\cite{wu2021flow}.
Researchers found that improving traffic flows in California can reduce highway carbon dioxide emissions by up to nearly 20\%; this is particularly significant given that transportation accounts for approximately a third of all US carbon dioxide emissions~\cite{barth2008real}.

\paragraph{Merging}
The case of AVs merging onto a highway with traffic is an appropriate task for investigating how we can scale supervision for a number of reasons. 
First, it is a common occurrence that can pose difficult control and coordination challenges for AVs~\cite{jula2000collision, zhou2016impact}.
Given the speed at which freeway merges occur, this task can result in deadly crashes.
Indeed, successful merging may sometimes rely on implicit communication between drivers or even non-traditional explicit communication, such as hand waves.
\textcolor{black}{For these reasons a `sufficient level' of safety in this paper refers to that which capable and attentive human drivers provide. This is achieved via the use of human supervision.} 

Merging also presents unique challenges in mixed autonomy settings. 
Were a system to consist entirely of AVs, a solution may exist via decentralized or centralized coordination, but that is much more difficult---if not impossible---for the foreseeable future, in which human drivers still populate the roadways~\cite{rios2016automated}.
Due to its relative frequency, unpredictability, and potential for impact across entire traffic systems, highway merging has been a particularly important and popular maneuver to study~\cite{chen2020hierarchical, rios2016survey, zheng2019cooperative}.

\paragraph{Other applications}

The scalable supervision framework, concepts, and formalization presented here are adaptable to other autonomous machine settings. 
Within mixed autonomy traffic environments, one could consider scalable supervision as a means of supporting dedicated AV lanes on freeways. 
These lanes have been shown in simulation to improve traffic efficiency while also worsening on-road safety~\cite{yu2019impact}.
The potential speed difference between AV-only and standard lanes present a merging problem similar to the case studied here. 
Scalable supervision would also be useful for the case of accident-prone `hotspots' (construction sites, school zones, etc.) where concerns arise about the potential behavior of AVs~\cite{erdogan2008geographical}. 
Each of these is a practical instance of the scenario studied here.

Beyond traffic scenarios, reachability-based supervision scaling that enforces safety constraints \textcolor{black}{may be} applicable to domains as varied as warehouses, ports, airports, satellite orbits, and disaster response.
\textcolor{black}{Our work can be situated among these domains along the axes of spatial scale and reliability requirements.
In robotics applications with human supervisors, the spatial scale is on the order of meters or less, but reliability requirements and response time conditions are not as stringent~\cite{hoque2022fleet, sheridan1986human}.
In human supervision of (increasingly autonomous) air and maritime traffic applications, the reliability requirements are high, but the spatial scale is on the order of tens of kilometers or more, corresponding to a larger temporal scale~\cite{kirwan2001role, van2020meaningful}. 
Our work seeks to improve supervision of autonomous agents in a direction of both high reliability (up to 99.9999\%) and meter-scale spatial resolution. 
Note that application to other domains requires computation of the relevant reachable sets, but---as we will show---conservative estimates of such sets in many of these applications may provide useful approximations that are more computationally efficient.} 

More broadly, \textcolor{black}{policies learned via RL have demonstrated superhuman performance in domains from board games and video games to high-altitude balloon control and simulated fighter jet combat~\cite{silver2017mastering, bellemare2020autonomous, hambling2020ai}.
However, deep RL methods lack the convergence guarantees of more traditional RL approaches, and thus cannot be trusted in safety-critical settings~\cite{bacci2020probabilistic, bouton2019reinforcement}.} 

\textcolor{black}{This work is similarly relevant to imitation learning, another approach to AV deployment that struggles to reliably maintain safety~\cite{le2022survey}. While progress has been made on various imitation learning challenges, in practice humans remain a fail-safe~\cite{kelly2019hg}.}

\subsection{Related Work}
This research draws from interesting and important related work in at least three areas of the literature: human-artificial intelligence (AI) teaming, supervision and cognitive models, and reachability analysis. 

\paragraph{Human-AI teaming}

Relevant work includes that from the human-robot interaction subfield in multi-agent systems, which explores how to reduce a single human operator's cognitive load from scaling linearly with the number of robots controlled. 
A significant portion of this work investigates the possibility for well-designed interfaces to reduce task burden or for learned models of human preferences to automate human-machine control handoff (also called, sliding, shared, or adjustable autonomy)~\cite{drew2021multi}.
Our work expands this effort by considering the shared autonomy situation in which strict safety guarantees must be maintained, and extends this line of work into the mixed autonomy traffic setting.

Research such as that described in~\cite{dahiya2021scalable} seeks to optimize the allocation of limited human assistance (via a decision support system) to multiple robots performing tasks in a given environment. 
The machines in these works benefit from human assistance, but do not exhibit the strict safety-critical requirements of AVs in traffic. 
As such, these works also place greater emphasis on computing and comparing the various assistance permutations. 
Such comparisons are not necessary given our binary approach to safety and strict safety requirements (i.e., no suboptimal safety configurations are permitted). 

\paragraph{Scalable supervision and cognitive models}
To the best of our knowledge, previous work concerned with decreasing the time burden of human supervision of autonomous tasks tends to be directed at empirical studies of a single human monitoring multiple UAVs (unmanned aerial vehicles)~\cite{cummings2007automation, cummings2007predicting, kidwell2012adaptable}.
Relatedly, human-robot interaction theory developed in the early 2000s considers `fanout,' the number of robots that a human can effectively control~\cite{cummings2007predicting, kidwell2012adaptable, humann2019human, conesa2015distributed, olsen2003metrics, olsen2004fan}.
While philosophically aligned with our work on some levels, in the AV merge setting we isolate a single task to analyze and face safety challenges unique to the mixed autonomy setting. 
Additionally, the mixed autonomy AV setting considered here is arguably higher-stakes than a typical UAV setting, as it includes human drivers operating in the same vicinity; crashes would more likely involve human fatalities.

\textcolor{black}{Our work can be seen as complementary to the research described in~\cite{hampshire2020beyond} and~\cite{daw2019beyond}, which also applies air traffic control principles to human, remote supervision of AV fleets.
Our reliability metric parallels the `exceedance probability' in~\cite{daw2019beyond}, but whereas that work's queuing analysis assumes batch arrivals of simulation-generated tasks upon which humans provide input (but not direct control), supervision tasks in our case are assigned to individual remote operators who control the vehicle for the entire task duration.
More broadly, our approach to defining the conditions requiring supervision is compatible with these works, which assume AVs can predict such scenarios.
Finally, these works do not consider how AV cooperation can facilitate supervision scaling.}


A closely related work is that on scaled autonomy by Swamy, et al., in which the authors investigate how to assist an operator in selecting which robots in a fleet most require teleoperation~\cite{swamy2020scaled}.
Our work differs by seeking to minimize the supervisor's active control time (the teleoperator in that work is never idle) and the mixed autonomy setting at the heart of our work must maintain strict safety standards. 

Broader research in this direction has considered a single machine supervisor observing vehicles with known paths~\cite{ahn2020robust} and algorithmic supervision of human drivers---in which a machine overrides human driving controls when the human is behaving in ways known to be unsafe~\cite{altche2017algorithm}.


\paragraph{Reachability analysis}
At its most basic, reachability analysis is about identifying the set of states that a dynamical system could enter---the `reachable states' or `reachable set'---given all admissible inputs and parameters~\cite{althoff2020set}.
A variety of methods exist for doing so; a common trade-off among these is that between approximation and computational inefficiency~\cite{kurzhanski2000ellipsoidal, liebenwein2018sampling}. 
Methods such as sampling-based approaches may require less computation, but they are not fully conservative, meaning they are not guaranteed to compute the reachable set~\cite{liebenwein2018sampling}.
For safety-critical applications, such approximations may be unacceptable. 

One method for computing the exact reachable set is Hamilton-Jacobi (HJ) reachability. 
HJ methods can handle nonlinear system dynamics and allow for formal treatment of bounded disturbances, but suffer exponential computational complexity in terms of the number of state variables~\cite{bansal2017hamilton}.
Fortunately, the vehicle dynamics and merging problem structure allow us to adopt a fully conservative approach while avoiding the computational complexity.
Our proposed reachability-based method calculates the exact reachable set for each vehicle in $O(1)$ time. 
However, we note that utilizing other mechanisms of computing conservative reachable sets (such as Hamilton-Jacobi reachability) could be used in future work generalizing to more scenarios~\cite{bansal2017hamilton}. 
Other work that seeks to integrate reachability analysis and AVs relaxes the conservatism of their approaches by making assumptions about human driving behaviors~\cite{bahati2020multi}, or by not accounting for the full range of system dynamics~\cite{leung2020infusing}.
The need remains for an approach that maintains strict safety guarantees. 




\section{Formalizing Scalable Supervision}


\begin{table*}[t]
  \caption{\textcolor{black}{Vehicle types and key properties. All AVs are assumed to allow for remote supervision and control; the connectivity referenced in the table (`unconnected,' `connected') refers to an AV's ability to communicate with other AVs in the system, in particular the merging AV.}}
  \label{tab:vehicle_properties}
  \centering
  \begin{tabular}{c|c|p{3.3cm}|p{5cm}|p{2.9cm}}
     &  & \multicolumn{3}{c}{\textbf{Properties}} \\ \cline{3-5}
    \textbf{Vehicle Type}  &  \textbf{Acronym} & Requires remote supervision \newline for dangerous merges & Shares near-term trajectory w/ merging AV \newline (resulting in truncated reachability zone) & Adjusts behavior to \newline accommodate AV merges \\ \hline\hline
    Human Vehicle  & HV  & &  &   \\ 
    Unconnected AV & UCAV & \hfil \checkmark & &   \\ 
    Noncooperative Connected AV & NCAV & \hfil \checkmark & \hfil \checkmark & \\ 
    Cooperative Connected AV & CCAV & \hfil \checkmark & \hfil \checkmark & \hfil \checkmark \\ 
  \end{tabular}
\end{table*}

In this section, we provide a formalization of scalable supervision.

\subsection{Formalizing Scalable Supervision}
\label{subsec:formalizing}
At its core, the two goals of scaling supervision are to reduce the number of necessary human interventions and to reduce the time required for those that must occur.  
One way to do this (our approach) is by utilizing reachability analysis to identify when a merging AV is in danger of colliding with another vehicle in the system, and activate human control of the AV \textcolor{black}{only} at that time. 
\textcolor{black}{This prevents unnecessary human interventions (e.g., for a merge when no other vehicles are around), and only calls for supervision at the moment transition of control is necessary, achieving the two goals above.} 
It is also important to note that `safety' in this paper is not synonymous with crash-free behavior. 
\textcolor{black}{This aligns with the operational reliability perspective on safety described above.}
Since the safety backup for the AV is a human driver, the ultimate safety of the system remains subject to human fallibility. 

To understand how our approach works, first imagine an AV on an on-ramp merging into a highway on which a human vehicle (HV) is already driving. 
The objective is to prevent the two vehicles from colliding during the AV's merge, so the reachability question, simply put, is whether it is possible for both the AV and HV to get to the merge point over some finite time horizon $t$.
The intuition here is that a supervisor need not assume control of the AV if a potential collision is far in the future; it is enough to assume control before the potential crash but with enough time to avoid it comfortably. 
We can vary $t$ to be more or less cautious in any given scenario.
In at least one setting humans required 5-8 seconds to safely assume control of a vehicle; this is the time window we adopt in generating results~\cite{johns2016exploring}.

The reachability question naturally decomposes into two independent subquestions: (1) Can the HV reach the merge point? And (2) Can the AV reach the merge point? 
That is, for each vehicle, is the distance it would travel if it were to apply its maximum acceleration over the horizon $t$ greater than or equal to the distance between that vehicle's current position and the merge point?
If the answer to \textit{both} subquestions is `yes,' then human supervisor control of the AV must be activated to ensure \textcolor{black}{human-level} safety. 
See Figure~\ref{fig:reach_conditions} for a visual representation.

We can write this more formally: the horizon $t_{sup}$ within which the supervisor \textcolor{black}{must assume control of a merging AV for safety} (denoted by the $AV$ subscript) is 
\textcolor{black}{
\begin{multline} \label{eq:1}
    t_{sup} = \min \{t \mid t \geq 0, (d_{HV}(v_{HV, 0}, a_{HV,max}, t) \geq d_{m, HV}) \\ \wedge (d_{AV}(v_{AV, 0}, a_{AV,max}, t) \geq d_{m, AV})\},
\end{multline} 
}
where $d_i(v_{i,0}, a_{i,max}, t)$ is the maximum distance vehicle $i$ can travel in time horizon $t$ and is a function of that vehicle's initial velocity $v_{i,0}$ and its maximum acceleration $a_{i,max}$, $d_{m,i}$ is the distance between vehicle $i$ \textcolor{black}{(either an HV or an AV)} and the merge point $m$, \textcolor{black}{and the HV is on the highway onto which the AV is merging}.
The road networks considered here are limited to one-lane roads and thus considering a vehicle's maximum forward distance is sufficient.
This amounts to the most conservative reachability for the given time horizon: there is no location that the vehicle could reach over $t$ time that is not within its reachability zone. 
Note that this maintains safety guarantees even when generalizing to cases in which 2-dimensional motion is considered (such as by adding lane changes).
The maximum distance reachability formulation still provides a conservative reachable zone because side-to-side motion only can reduce the distance along the road that the vehicle travels, rather than increase it. 

\subsection{Scenario}

Because this work focuses on merging, the analysis begins by considering a large rotary traffic network, which approximates a highway with multiple on-ramps and off-ramps. 
The ring road at the core of the structure is a well-studied setting in traffic literature and has been shown to mimic traffic congestion patterns that might occur on an infinite roadway~\cite{sugiyama2008traffic}.
The ring therefore emulates a highway.
Merging vehicles enter the system via the on-ramps, merge into the ring road, and then exit via off-ramps
(see Figure~\ref{fig:merge_schematic}).

Since this is a mixed autonomy system, we consider both AVs and human vehicles (HVs). 
Acceleration profiles are parameterized via the widely used Intelligent Driver Model (IDM), which has been shown to emulate the actual behavior of human drivers~\cite{treiber2000congested, treiber2013traffic}.
\textcolor{black}{We first present analysis for mixed autonomy traffic with unconnected AVs (UCAVs). These are AVs that cannot communicate with other AVs in the system (but still allow for remote supervision).}
Next, we consider noncooperative connected AVs (NCAVs)---also using IDM---to investigate the supervision scaling implications of connected AVs.
These are modeled such that they can communicate with other connected AVs in the system, but do not alter their behavior to accommodate merging connected AVs.
We will see that they allow for linear improvements in supervision scaling. 
Lastly, we analyze cooperative connected AVs (CCAVs):
these represent supervision-aware AVs that not only can communicate with other connected AVs in the system, but also seek to assist merging connected AVs to effectively join highway traffic while reducing the need for human supervision.
Via analytical results we show how the teaming of CCAVs enables orders-of-magnitude improvement in supervision requirements.
\textcolor{black}{See~\Cref{tab:vehicle_properties} for a summary of the vehicle types and their properties.}
\textcolor{black}{Note that all AVs (including UCAVs) are assumed to allow for remote supervision and control.}

\section{Theoretical Analysis}

By applying reachability analysis to the problem of AV merging, a number of bounds can be derived for the settings in which we are interested. 
To analytically quantify the risk and scalability of supervision, we take a queuing theoretic approach in this section to analyze long-term steady-state conditions.

This analysis begins with a description of an upper bound on the supervision requirements for the scenario in which a group of remote supervisors manage merges for an arbitrarily large number of on-ramps, as well as a closed-form expression for calculating the probability that AVs need supervision but cannot receive it. 
This is of particular interest because it allows for the characterization of the risk associated with a system with a fixed number of supervisors.
Conversely, it allows one to calculate the number of supervisors necessary to achieve a desired level of supervision safety. 
After the provision of the upper bound and analysis of the UCAV and NCAV settings, an analysis of the `typical case' is provided, wherein the expected supervision scaling gains via (supervision-aware) CCAVs are described. 

\textcolor{black}{Let $C$ be a \textit{potential} conflict event. 
Here this refers to the joint event that a merge point falls within the reachable zone of at least one vehicle in the ring (event $C_{\text{in}}$) and that the same merge point falls within the reachable zone of a vehicle outside the ring that is attempting to merge in (event $C_{\text{out}}$).
$C_{\text{in}}$ and $C_{\text{out}}$ thus correspond to the subquestions outlined in~\Cref{subsec:formalizing}.
In different settings one may consider alternative dangerous situations and methods other than reachability analysis to indicate them.
One necessary component of the remote supervisor analysis is a characterization of $P(C_{\text{in}})$, that is, the probability that the merge point falls within the reachable zone of at least one vehicle on the ring.}
The analysis further below investigates this by considering how an upper bound on $P(C_{\text{in}})$ varies in the settings in which the in-ring vehicles are entirely HVs, a mix of HVs and NCAVs, and finally a mix of HVs and (supervision-aware) CCAVs.
Note the NCAVs are \textit{not} antagonistic; they do not seek to inhibit merges, but simply do not alter their path to accommodate them. 

It is important to first note that we can rewrite~\Cref{eq:1} using kinematics.
We allow each vehicle to take any non-negative velocity $v \in [0, \infty)$ and only assume finite vehicle acceleration $a_i \in [a_{i, min}, a_{i, max}]$. 
\textcolor{black}{This provides conservative, physics-informed safety: it is impossible for a vehicle to cover more distance than it does when maximally accelerating over the time horizon's duration.} 
Thus, preserving this view of reachability requires accounting for the worst-possible case, and the reachable set for vehicle $i$ at position $0$ is $[0, d_{i}(v_{i,0}, a_{i, max}, t)]$.

This conservative approach to reachability has the ancillary benefit that vehicle $i$'s reachable distance in the given time horizon can be written simply: 
\begin{equation} \label{eq:d_kinematics}
    d_{i}(v_{i,0}, a_{i, max}, t) = v_{i,0} t + \frac{1}{2} a_{i, max} t^2
\end{equation}

\textcolor{black}{While this work uses IDM values for the vehicles' maximum accelerations,} in practice acceleration profiles can be derived from data for determining a pragmatic upper bound for a given setting\textcolor{black}{~\cite{rakha2004vehicle}}. 
\textcolor{black}{Similarly, one could consider decoupling $t$ in the above from vehicle dynamics. 
For example, it (or even the reachable distance itself) could be provided by a prediction system that estimates vehicles' forward progress~\cite{brown2020taxonomy}.
This work uses the kinematics-based approach to maintain strict safety requirements, but we expect much of the analysis to apply to less conservative approaches.
In light of this, in the statements below these components are left as variables (although sometimes omitted from explicit inclusion for ease of exposition). 
Specific values are only applied during the generation of the numerical results and are sourced from relevant literature.} 
With this in mind, we turn to the challenge of characterizing the multiple-supervisor, multiple-merge case. 

\subsection{Multiple supervisors monitoring an AV fleet}

We consider a future setting in which a team of human supervisors located remotely could monitor a fleet of AVs and assume remote control if necessary.
A natural question would be: how many supervisors are necessary? 

\begin{theorem} \label{theorem_queue}
Suppose we have $k$ remote supervisors and $n$ on-ramps on which $AV$s appear and trigger the on-ramp supervision condition with independent arrival processes $Poisson(\lambda_1), ..., Poisson(\lambda_n)$. Suppose the service rate of each remote supervisor follows $Exp(\mu)$. 

The fraction of AVs that require supervision but cannot immediately receive it (and thus go unsupervised) is given by
\begin{equation}
\label{eq:pm}
    \begin{split}
        P_k = \frac{(\lambda/\mu)^k / k!}{\sum\limits_{i=0}^{k} (\lambda / \mu)^i / i!}
    \end{split}
\end{equation}
where $\lambda = \sum\limits_{j=1}^{n}\lambda_j P(C_{\text{in}})$, $P(C_{\text{in}})$ is the probability that an in-ring vehicle will trigger its supervision condition, and $\lambda < \mu$ in order to have a valid steady state probability.
\end{theorem}

\begin{proof}
Allow different on-ramps to have different arrival rates $\lambda_1, ..., \lambda_n$.
Because the supervision tasks are allocated to a centralized group of remote supervisors, the problem reduces from a tasking situation involving $n$ separate queues to a single-queue tasking case.
Given \textcolor{black}{the additivity of independent Poisson random variables} (see~\Cref{poisson_random_var_sum_statement} in the appendix), the arrival rate of AVs requiring supervision follows $Poisson((\sum\limits_{j=1}^{n}\lambda_j)P(C_{\text{in}}))$.
\textcolor{black}{Poisson arrivals are known to successfully model true vehicle arrivals and are commonly used highway on-ramp arrival modeling~\cite{rengaraju1995vehicle, sarla2020performance, mirchandani2006analytical, ding2020penetration}.} 
\textcolor{black}{Exponential service rates are used in modeling air and roadway traffic supervision tasks in~\cite{hampshire2020beyond} and~\cite{zhang2018analysis}.} 
Set $$\lambda = (\sum\limits_{j=1}^{n}\lambda_i)P(C_{\text{in}}).$$

Let $X$ be the random variable denoting the number of AVs requiring control at an arbitrary point in time. 
Without loss of generality, assume $\lambda_1 = ... = \lambda_n = \hat{\lambda}/ n$ for some constant $\hat{\lambda}$, then the AVs arriving on the on-ramp trigger the supervision condition following the distribution $X \sim Poisson(\sum\limits_{i=1}^{n}\hat{\lambda} / n) = Poisson(\hat{\lambda})$ (again, see~\Cref{poisson_random_var_sum_statement} in the appendix).
Let $Y$ be a random variable denoting the number of AVs that need to be supervised.
We know $Y | X \sim Binomial(X, P(C_{\text{in}}))$ where $P(C_{\text{in}})$ can be obtained from the in-ring vehicle’s reachability condition and will be derived further below.

By~\Cref{poisson_binomial_statement} (see Appendix), we have that $Y \sim Poisson(\hat{\lambda} P(C_{\text{in}}))$ is the arrival process for AVs that need to be supervised.
Then, the problem reduces to an $M/M/k$ queue with arrival process $Poisson(\hat{\lambda}P(C_{\text{in}}))$, service rate $Exp(\mu)$ for each supervisor, and finite capacity $k$.
The finite capacity is due to the fact that when all supervisors are busy, any additional on-ramp AVs requiring supervision will be rejected immediately. 
These cases cannot wait on the queue for future service as their needs are immediate; that is, to maintain safety, once the on-ramp and in-ring conditions are triggered for a given merge, a supervisor must immediately supervise the merge. 
Thus, the rejected cases that the pool of $k$ supervisors cannot immediately service represent dangerous situations. 

For an $M/M/k$ queue with no waiting space, both the steady state probability and loss formula are known~\cite{gross2008fundamentals}.
The expression for the steady state probability is
\begin{equation}
    \begin{split}
        P_q = \frac{(\lambda / \mu)^q / q!}{\sum\limits_{i=0}^{k} (\lambda / \mu)^i / i!}
    \end{split}
\end{equation}
where $q \in\{0,\dots,k\}$. Then, $q$ represents the number of AVs requiring supervision, $k$ is the total number of supervisors (as previously defined), the arrival rate is $\lambda = \hat{\lambda}P(C_{\text{in}})$, and $\mu$ is kept as the service rate.
The loss formula here represents the fraction of AVs that require supervision but cannot immediately receive it (and thus go unsupervised), and the result follows.


\end{proof}

Based on Equation~\eqref{eq:pm}, the expected number of supervisors needed is the smallest $k$ such that $P_k \leq \delta$ for some risk tolerance of interest $0 < \delta << 1$.
\textcolor{black}{Various examples of the relationships encoded in~\Cref{eq:pm} are visualized in~\Cref{fig:geneqn_num_sup_num_avs_single} and~\Cref{fig:geneqn_confidence_single} for a mixed HV/UCAV traffic network setting (described below) with fixed inflow and service rates.
Figures~\ref{fig:geneqn_num_sup_array},~\ref{fig:geneqn_avs_per_sup_array}, and~\ref{fig:geneqn_confidence_array} in the Appendix present results for a variety of inflow and service rates; these illustrate how the trends shown in Figures~\ref{fig:geneqn_num_sup_num_avs_single} and~\ref{fig:geneqn_confidence_single} hold more generally.}
In short, these answer the question, ``How many supervisors do I expect to need?" in the UCAV case, assuming the questioner can define their risk threshold $\delta$.

In the above, observe that---due to the functional reduction of the $n$ inflows to a single task queue serviced by the $k$ supervisors---the exact number of on-ramps does not directly affect the steady state probability or loss formula, except insofar as the individual arrival rates $\lambda_1, ..., \lambda_n$ contribute to the global arrival rate.

\begin{figure*}[tbhp]
    \centering
    \includegraphics[width=0.85\textwidth]{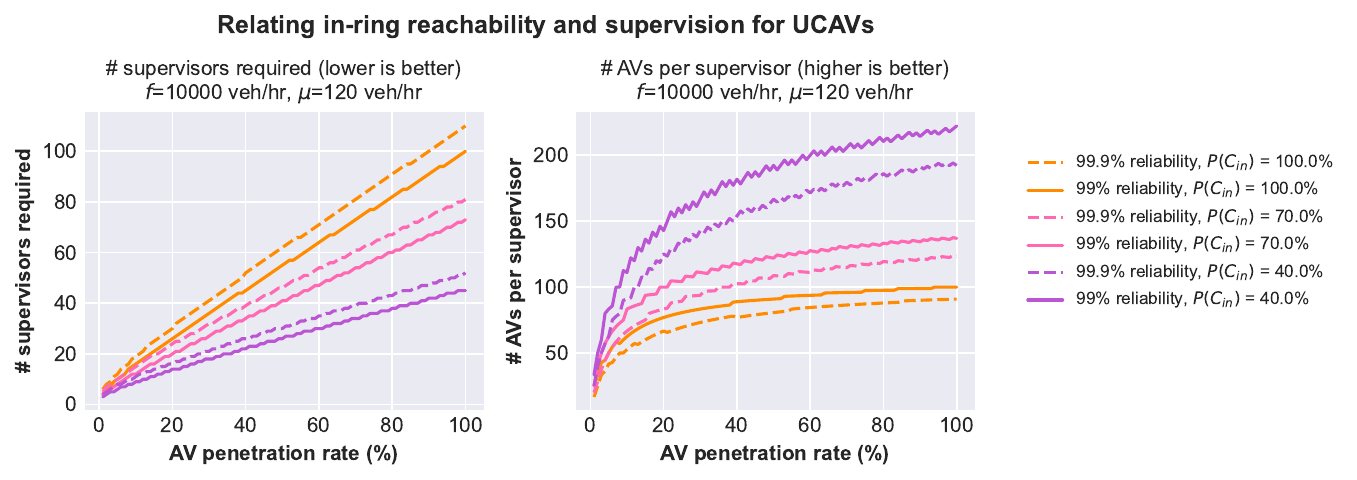}
    \caption{An illustration of the relationships in~\Cref{eq:pm} for the `baseline' case with unconnected AVs (UCAVs).
    $f$ is the total vehicle inflow rate.
    \textcolor{black}{`AV penetration rate’ refers to the percentage of vehicles in the system (including inflows) that are autonomous.}
    Note that as the AV penetration rate grows, so does the number of supervisors needed to maintain a given safety threshold. This is because the number of merging AVs grows with greater penetration rates. Also note how altering $P(C_{\text{in}})$ affects the plots through its effect on $\lambda$.
    \textcolor{black}{The $\mu$ value represents monitors that can supervise a merge in an average of 30 seconds. The selection rationale for the values of $\mu$ and $f$ is described in~\Cref{subsec:sup_req_for_traffic_network}.}
    } 
    \label{fig:geneqn_num_sup_num_avs_single}
\end{figure*}

\begin{figure}[t]
    \centering
    \includegraphics[width=0.35\textwidth]{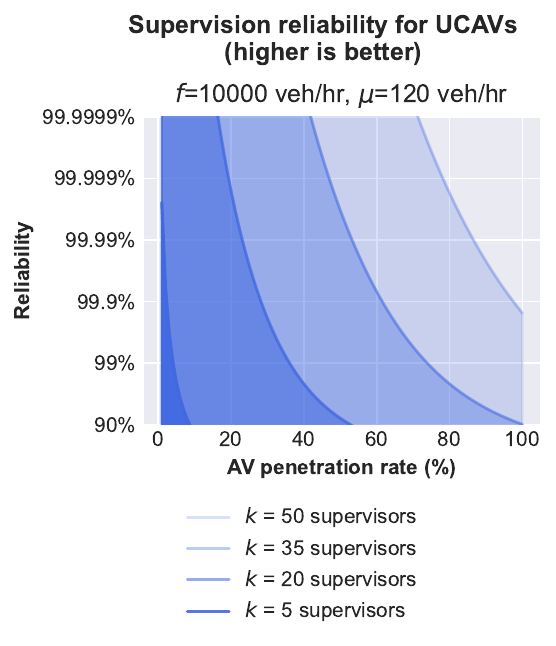}
    \caption{\textcolor{black}{A visualization of the reliability thresholds in the UCAV case, where the `reliability' metric is calculated as $1 - P_k$. For example, 99.9\% reliability means that of 1000 AVs requesting a supervisor due to reachability conditions being violated, 1---on average---will not receive it. This does not necessarily mean a crash occurs, but that one is possible and no human supervisor is available to monitor the situation. As AV penetration increases, the system's ability to maintain adequate supervision drops monotonically. Note the logarithmic y-axis.}}
    \label{fig:geneqn_confidence_single}
\end{figure}

\textcolor{black}{We recognize that the exponential distribution of service times inherent to a Poisson process can present unrealistic scenarios (service rarely takes zero time).
Indeed, this is a known limitation when using Poisson processes for modeling physical systems.
A truncated exponential distribution can mitigate this issue \textcolor{black}{and can be handled explicitly using discrete event simulation.} 
The numerical analysis below adopts rates $\mu$ substantially larger than the reachability horizon to implicitly account for this shift.} 

Finally, note the importance of the \textit{in-ring reachability} $P(C_{\text{in}})$ in determining whether an AV requires supervision. 
This follows naturally from the problem formalization, in which both the in-ring and on-ramp supervision conditions (reachability conditions) must be met in order for a merging AV to require supervision. 
The following sections provide insight into this term.
\textcolor{black}{The above equations are provided as functions of $P(C_{\text{in}})$, among other variables.
In the following lemmas, corollaries, and theorem, we consider various cases where we can compute $P(C_{\text{in}})$ analytically to investigate how the system's supervisability under various conditions.} 

\subsection{In-ring reachability for \textcolor{black}{mixed HVs and unconnected AVs (UCAVs)}}

\begin{lemma} \label{lemma_HVs}
Given a single-lane ring road of circumference $c$ with $R$ \textcolor{black}{unconnected} vehicles \textcolor{black}{(HVs or UCAVs)} distributed uniformly at random (but not necessarily independently) along the length of the ring, the probability that an arbitrary fixed point on that ring is reachable over time horizon $t$ by a vehicle is bounded above as follows:
\begin{equation}
    P(C_{\text{in}}) \leq \frac{\sum\limits_{i=1}^{R} d_i(v_{i,0}, a_{i,max}, t)}{c}
\end{equation}
where $d_i(v_{i,0}, a_{i, max}, t)$ is the reachable range for vehicle $i$ over the time horizon.
\end{lemma}

\begin{proof}
Suppose $X$ is the fixed merge point \textcolor{black}{and the $R$ unconnected vehicles are composed of $R-S$ HVs and $S$ UCAVs}. Let $H_i$ be the asymmetric distance (following traffic flow) between the $i^{th}$ in-ring \textcolor{black}{unconnected vehicle} and $X$. Then the event
\begin{equation} \label{union_cond_eqn_HV}
    \begin{split}
        C_{\text{in}} &= \cup_{i=1}^{R-S}\{\text{HV $i$ triggers supervision cond.}\} \\
        &\qquad \cup \left(\cup_{j=1}^{S}\{\text{UCAV $j$ triggers sup. cond.} \}\right) \\
        &= \cup_{i=1}^{R}\{H_i \leq d_i(v_{i, 0}, a_{i,max}, t)\}
    \end{split}
\end{equation}
where the first equality comes from the fact that the event $C_{\text{in}}$ happens if any of the \textcolor{black}{unconnected vehicles} triggers the supervision condition, and the second equality comes from the reachability analysis for each \textcolor{black}{HV and UCAV}.
\textcolor{black}{Because the UCAVs cannot communicate with other vehicles in the system, merging AVs must adopt the same conservative approach applied to HVs. Indeed, from the merging AV's perspective, HVs and UCAVs are functionally indistinguishable.}

Applying a union bound to the joint probability event in~\Cref{union_cond_eqn_HV}, we have
\begin{equation}
    \begin{split}
        P(C_{\text{in}}) & \leq \sum\limits_{i=1}^{R}P(\{H_i \leq d_i(v_{i, 0}, a_{i,max}, t)\}) \\
        &= \sum\limits_{i=1}^{R}\frac{ d_i(v_{i, 0}, a_{i,max}, t)}{c}
    \end{split}
\end{equation}
where the last equality comes from our assumption that each vehicle's location is distributed uniformly at random along the length of the ring, i.e. $H_i \sim U([0, c])$.
\textcolor{black}{This uniform distribution is a product of the infinite uniform roadway that the ring road simulates~\cite{sugiyama2008traffic}; even if traffic occurs, it is equally likely to do so at each point on the road. 
Furthermore, because this work focuses on merging into free-flowing highways, nonuniform congestion is beyond our scope.} 
However, note this does not require each vehicle's location be independent of other vehicles' positions; \textcolor{black}{this is investigated further below}. 

\end{proof}

The uniform distribution assumption above may be satisfied both by (a) a situation in which traffic is flowing freely around the ring, and (b) a setting with stop-and-go traffic where the congestion is equally likely to occur at any point within the ring.
A case not covered by the lemma is one in which congestion routinely occurs around the merge point, but this case is beyond the current scope of work given our focus on safety in high-speed merges.

\subsection{Joint in-ring and on-ramp reachability for \textcolor{black}{mixed HVs and unconnected AVs (UCAVs)}}
\begin{lemma} \label{lemma_HVs_joint_prob}
Given a single-lane ring road of circumference $c$ with $R$ \textcolor{black}{unconnected vehicles (HVs or UCAVs)} distributed uniformly at random (but not necessarily independently) along the length of the ring, and a merging vehicle $q$ distributed uniformly at random along an on-ramp of length $l_{\text{on-ramp}}$, the joint probability that the on-ramp's merge point into the ring road is reachable over time horizon $t$ by both an in-ring vehicle and the on-ramp vehicle is bounded above as follows:
\begin{equation}
    \begin{split}
        P(C) =\ &P(C_{\text{in}}) \times P(C_{\text{out}}) \\ 
        \leq\ &\frac{\sum\limits_{i=1}^{R} d_i(v_{i,0}, a_{i, max}, t)}{c} \times \frac{d_{q}(v_{q,0}, a_{q,max}, t)}{l_{\text{on-ramp}}}
    \end{split}
\end{equation}
where $d_i(v_{i,0}, a_{i, max}, t)$ is the reachable range for vehicle $i$ over the time horizon $t$.
\end{lemma}

\begin{proof}
Without loss of generality, we assume a sufficiently long on-ramp, i.e. $l_{\text{on-ramp}} \geq d_i(v_{i,0}, a_{i, max}, t)$.

The above follows from analyzing the joint probability of independent events.
The rightmost term $P(reach_{\text{on-ramp}})$ is the probability that an arbitrary point on the on-ramp is within the reachable zone of merging vehicle $q$ over horizon $t$, and follows similar logic to that used for writing $P(C_{\text{in}})$.
\end{proof}

Consequently, this also relies upon the assumption that the likelihood of finding the merging vehicle in a given position is evenly distributed across the on-ramp.
This assumption is more tenuous in this case than it was in~\Cref{lemma_HVs} because in situations of interest (such as when traffic exists on the highway) the AV may have to yield.
This would cause the AV to spend a disproportionately large amount of time just before the merge point.
\textcolor{black}{A more thorough on-ramp analysis would also have to account for the interaction between AV behavior and HV responses; this is an interesting direction for future work.}

\subsection{In-ring reachability for mixed HVs and \textcolor{black}{noncooperative,} connected AVs (NCAVs)} 

Now consider how a mixed autonomy system with NCAVs may improve the situation. \textcolor{black}{NCAVs may communicate their near-term trajectories but are not assumed to alter their trajectories to avoid triggering supervision in the system.
That is, they do not actively accommodate merging vehicles.} 
\textcolor{black}{The benefit of this connectivity is that in-ring NCAVs' reachable zones can be reduced from where they \textit{might} be to where they \textit{will} be, along with a desired safety buffer.
This can have a substantial impact, as the kinematics-based reachability accounts for a vehicle's maximum possible acceleration and is thus likely to overestimate the true distance. 
This decreases the probability that they trigger the in-ring reachability condition. 
Further discussion of this point--with particular emphasis on the NCAVs' ability to predict their trajectories---is provided after the corollary's presentation.} 

\begin{corollary} \label{connected_av_cor}
Given a single-lane ring road of circumference $c$ with $R - S$ \textcolor{black}{HVs} and $S$ \textcolor{black}{N}CAVs distributed uniformly at random (but not necessarily independent of each other) along the length of the ring, the probability that an arbitrary point on that ring is reachable over time horizon $t$ by a vehicle is bounded above as follows: 
\begin{equation}
    P(C_{\text{in}}) \leq \frac{\sum\limits_{i=1}^{R-S} d_i(v_{i,0}, a_{i,max}, t) + \sum\limits_{j=1}^{S} l_j}{c},
\end{equation}
where $d_i(v_{i,0}, a_{i, max}, t)$ is the reachable zone for vehicle $i$ over the time horizon and $l_j$ is the length of NCAV $j$. Note that one may define $l_j$ to include a safety buffer. 
\end{corollary}

\begin{proof} 
Let $H_i$ be the asymmetric distance (following traffic flow) between the $i^{th}$ in-ring HV ($i=1, ..., R-S$) and the fixed merge point at the current time $t=0$. Let $\tilde{A}_j$ be the asymmetric distance between the $j^{th}$ in-ring NCAV ($j = 1, ..., S$) and the fixed merge point at future time $t=t_f$. Due to the connectivity, we know the near-term trajectory of the NCAVs and hence the location of the in-ring NCAVs up to future time $t=t_f$, so NCAV $j$ triggers supervision if its future location at time $t=t_f$ has distance less than \textcolor{black}{its length $l_j$ (i.e., if it is directly at the merge point).
$l_j$ could be expanded to include a safety buffer.} 
Additionally, the corollary's assumption says $H_i \sim U[0, c]$ and $\tilde{A}_j \sim U[0, c]$.

Then we similarly have the event $C_{\text{in}}$ as the event that any of the HVs or NCAVs trigger supervision:

\begin{equation} \label{union_cond_eqn_CAV}
    \begin{split}
        C_{\text{in}} & = \left(\cup_{i=1}^{R-S}\{\text{HV $i$ triggers sup. cond.}\}\right) \\
        &\qquad \cup \left(\cup_{j=1}^{S}\{\text{NCAV $j$ triggers sup. cond.} \}\right) \\
        &= \left(\cup_{i=1}^{R-S}\{H_i \leq d_i(v_{i, 0}, a_{i,max}, t)\}\right) \\
        &\qquad \cup \left(\cup_{j=1}^{S}\{\tilde{A}_j \leq l_j \}\right)
    \end{split}
\end{equation}

Applying a union bound to the joint probability event in~\Cref{union_cond_eqn_CAV} results in

\begin{equation}
    \begin{split}
        P(C_{\text{in}}) &\leq \sum\limits_{i=1}^{R-S}P(\{H_i \leq d_i(v_{i, 0}, a_{i,max}, t\}) \\
        &\qquad + \sum\limits_{j=1}^{S}P(\{\tilde{A}_j \leq l_j\}) \\
        &=\sum\limits_{i=1}^{R-S} \frac{ d_i(v_{i, 0}, a_{i,max}, t)}{c} + \sum\limits_{j=1}^{S} \frac{l_j}{c}
    \end{split}
\end{equation}

\end{proof}

\textcolor{black}{The $t_f$ term need not explicitly appear in these expressions because an NCAV's location is independent of the merge point, and thus $t_f$ is unnecessary for assessing probabilities (as opposed to individual instances).}

It is worth further examining connected AVs' ability to predict and communicate their trajectory.
\textcolor{black}{Trajectory planning is a well-studied problem in AVs~\cite{rasekhipour2016potential}.}
However, mixed autonomy settings present unique challenges to trajectory planning---while AVs can plan their own trajectories over a given time horizon in isolation, the human drivers on the road are unpredictable. 
How can AVs predict their own trajectories when sharing the road with human drivers? 

First, note that faster-than-expected HVs in this setting do not pose a problem for the connected AVs' trajectory planning; indeed, if they speed further ahead, the AVs have more space.
The instances which might be problematic are those in which an HV quickly slows. 

However, even here the problems subside upon further analysis. 
If an HV sharply brakes far ahead of the merge point, the connected AVs behind it are also far from the merge point, and thus do not pose a collision risk for the merging vehicle. 
If an HV sharply brakes far after the merge point, no connected AV near the merge point will be substantially affected, especially not immediately. 

Thus the instances of concern are further limited to those cases in which an HV rapidly brakes near the merge point when an on-ramp vehicle is about to merge. 
However, if the HV of concern is just prior to the merge point, note that it will have already triggered the supervisor, so any adjustment of a connected AV's propensity to trigger the supervision condition is redundant. 
(Recall that once a supervisor is triggered, it supervises the entire merge.)

The remaining case is that in which an HV has just passed the merge point when it rapidly decelerates. 
Yet the only way an HV manages to pass the merge point without triggering supervision is if the merge point is beyond the on-ramp vehicle's reachable zone during the entire time that same point is within the HV's reachable zone. 
Therefore, even in this final case of concern, if the HV behaves erratically, there remains the entirety of the reachability time horizon for the connected AV to respond---and likely significantly more, given that the merging vehicle's reachable zone assumes its maximum acceleration.

\textcolor{black}{Note that in these cases ultimate safety still depends on the ability of the human supervising the merging AV to respond to traffic conditions appropriately (as one would expect for any human driver).
The above analysis simply explains the in-ring NCAVs' ability to reduce the uncertainty in their reachability zones from where they might be to where they will be. 
This reduces the odds supervision requirements will be triggered while still preserving the time horizon for a supervisor to respond, should one be needed.} 
Of course, as discussed previously, unforeseen events can occur, but that is beyond the scope of this research focusing on the merge event.

\subsection{In-ring reachability for mixed HVs and \textcolor{black}{cooperative} (supervision-aware), connected AVs (CCAVs): worst case} 
\label{subsec:in-ring-reachability-worst-case}

Now consider how a mixed autonomy system with connected AVs that cooperate to avoid triggering supervision \textcolor{black}{(that is, supervision-aware CCAVs)} may improve upon the previous cases. 
\textcolor{black}{Recall that both CCAVs and NCAVs may communicate their near-term trajectories to merging AVs. However, whereas NCAVs do not alter their trajectories to accommodate merging AVs, CCAVs can alter their trajectory to assist AVs to merge without triggering human supervision.} 

We first consider a worst case improvement.
Then, in Section~\ref{subsec:in-ring-reachability-uniform-case} and Section~\ref{subsec:in-ring-reachability-induced-case}, we consider the improvement given more typical cases with the CCAVs.

\begin{corollary} \label{lemma_RLAV_worst}
Given a single-lane ring road of circumference $c$ with $R - S$ \textcolor{black}{HVs} and $S$ (supervision-aware) \textcolor{black}{CC}AVs distributed uniformly at random along the length of the ring, the probability that an arbitrary point on that ring is reachable over time horizon $t$ by a vehicle is bounded above as follows:
\begin{equation}
    P(C_{\text{in}}) \leq \frac{\sum\limits_{i=1}^{R-S} d_i(v_{i,0}, a_{i,max}, t)}{c},
\end{equation}
where $d_i(v_{i,0}, a_{i, max}, t)$ is the reachable zone for vehicle $i$ over the time horizon. 

\textbf{That is, the previous upper bound in~\Cref{connected_av_cor} is improved by dropping the $\sum\limits_{j=1}^{S} \frac{l_j}{c}$ term}.
\end{corollary}

\begin{figure*}[thbp]
    \centering
    \includegraphics[width=1.0\textwidth]{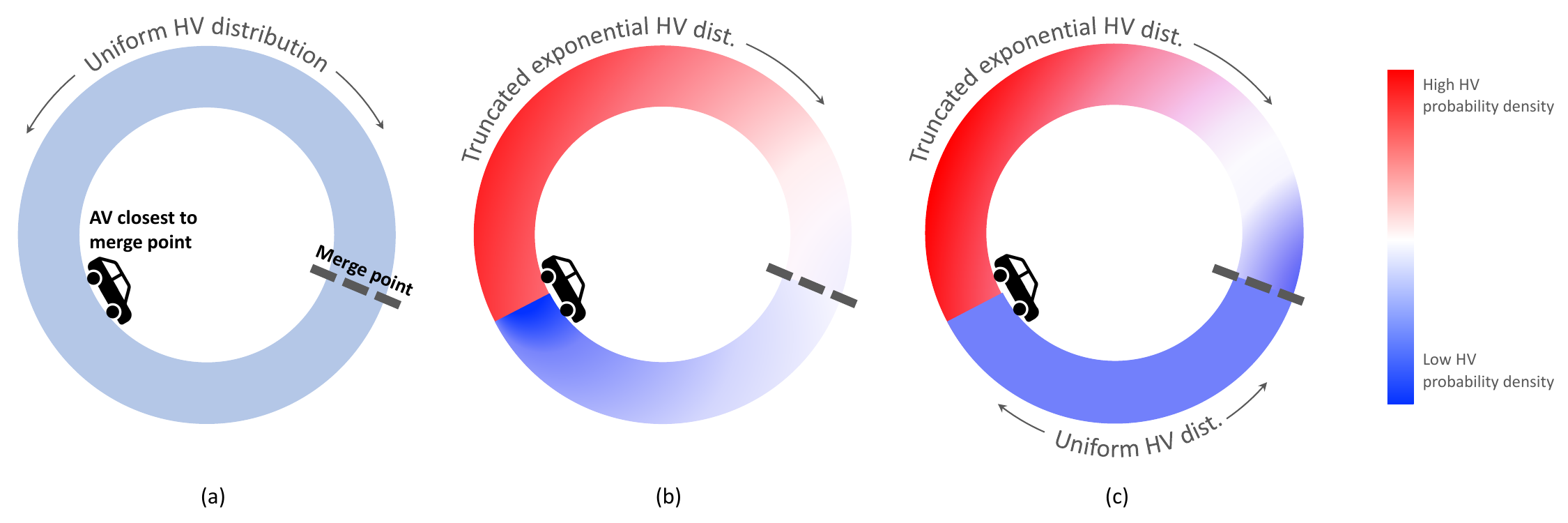}
    \caption{\textcolor{black}{Examples of three different HV distributions in the mixed HV/CCAV traffic setting. (a) is the case in Example 1, in which HVs are equally likely to be at any point in the ring, regardless of CCAV location(s). (b) considers a case from the induced traffic distribution scenario (Example 2), in which a CCAV forms a `platoon' of vehicles behind it, such that the distribution of HVs is most dense immediately after the CCAV and thereafter decays smoothly around the ring. (c) corresponds to Example 3. It recognizes that (b) is slightly adversarial: HVs in front of the lead CCAV (between the lead CCAV and the merge point), are more likely to be close to the merge point than far from it. Of course, given asymmetric traffic flows, it is unrealistic for a CCAV to determine the location of HVs in front of it. Therefore, (c) utilizes the exponential distribution for HVs behind the lead CCAV, while allowing a uniform distribution for HVs in front of it. Each subfigure here can be conceptualized as taking a vertical `slice' from its corresponding subfigure in~\Cref{fig:double_integration_uni_grad} and wrapping it around the ring road. 
    As such, each diagram here is a visualization of \textit{only one possible configuration} for each distribution; more general representations (i.e., the results if all vertical slices for each example were stacked next to each other) are shown in \Cref{fig:double_integration_uni_grad}.}}
    \label{fig:three_rings}
\end{figure*}

\begin{proof}
Again, let $H_i$ be the distance between the $i^{th}$ in-ring HV ($i=1, ..., R-S$) and the fixed merge point at the current time $t=0$. We assume $H_i \sim U[0, c]$.

 Given the fixed merge point and a perfect control of its trajectory during the length $\tilde{t}$ planning interval, so long as $\tilde{t}$ is sufficiently long, there exists a control input (sequence of accelerations over the planning horizon) such that the CCAV is not at the merge point at $t=t_f$ with probability 1. \textcolor{black}{For example, this could consist of slowing down or speeding up to create a gap for the merging AV. See more discussion of $\tilde{t}$ below.} 
 So the event $\{\text{CCAV j triggers supervision condition}\} = \emptyset$, i.e., does not occur. Hence the above upper bound improves from the previous one in~\Cref{connected_av_cor} with the $\sum\limits_{j=1}^{S} \frac{l_j}{c}$ term. 
 
 In the worst case, the order of HVs and CCAVs (with respect to the merge point) is adversarially distributed, such that (1) all the HVs are asymmetrically closer than all CCAVs to the merge point (when moving in the forward direction), and (2) all HVs have non-overlapping reachability zones. In this case the CCAVs cannot influence any HV's behavior at the merge point, so it is possible for any HV to trigger the supervision condition. Furthermore, the HVs' combined reachability zone is achieving its maximum coverage over the ring.

The event that an in-ring vehicle triggers supervision is thus due to the HVs:
\begin{equation}
    C_{\text{in}} = \cup_{i=1}^{R-S}\{\text{HV $i$ triggers sup. cond.}\}.
\end{equation}

The union bound gives 
\begin{equation}
    P(C_{\text{in}}) \leq \frac{\sum\limits_{i=1}^{R-S} d_i(v_{i, 0}, a_{i,max}, t)}{c}.
\end{equation}

\end{proof}

\begin{figure*}[t!]
    \includegraphics[width=1.0\textwidth]{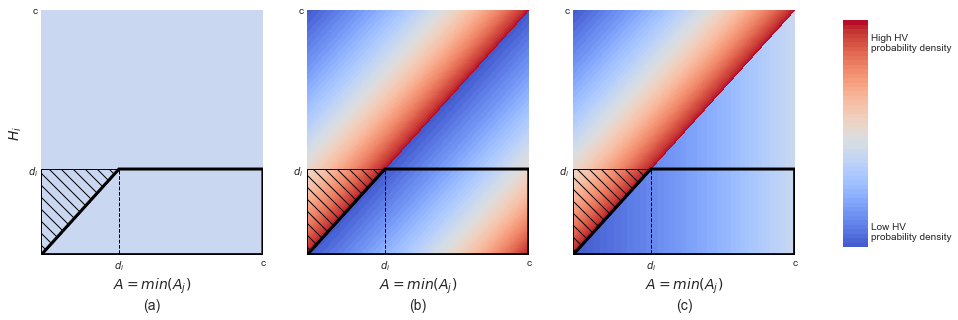}
    \caption{Visual representation of the human vehicles' conditional distributions $f_{H_i| A_{(1)}}(h|a )$ for the three examples described in Sections~\ref{subsec:in-ring-reachability-uniform-case} and~\ref{subsec:in-ring-reachability-induced-case}, along with a visualization of the double integration from~\Cref{eq:absolute_prob}. The two axes represent the random variables $H_i$ and $A_{(1)} := \min A_j$. The bold black border encapsulates the area captured by the double integral, and the shaded triangle represents the portion of the ring for which HVs that would otherwise pose a threat to a merging vehicle will be blocked. The colors represent the varying probability density. (a) The probability in the uniform distribution case (Example 1). The even tone indicates an equal likelihood of HVs at all ring positions relative to the blocking CCAV. (b) The conditional probability in the nonuniform `platoon' distribution case when HVs follow closely behind the CCAV (Example 2). The red indicates a higher likelihood of HVs immediately behind the blocking CCAV and the blue indicates a lower likelihood of HVs immediately preceding it. Thus in this second case the CCAV produces a greater supervision scaling effect in the red zone, but is also disadvantageous in that the HVs between the CCAV and the merge point are more likely to be closer to the merge point. (c) The HVs' nonuniform conditional distribution as described in~\Cref{sec:sa_case3}.
    }
    \label{fig:double_integration_uni_grad}
\end{figure*}

Note that $\tilde{t}$ is distinct from the reachability time horizon $t$, and in practice would likely be much smaller. 
$\tilde{t}$ simply corresponds to the case in which the planning interval is too short for the CCAV to avoid the merge point.
For example, if a merging AV appears on the on-ramp when an in-ring, supervision-aware CCAV is just before the merge point and moving quickly, there may not exist sufficient time for the in-ring CCAV to brake---and thus block any HVs behind it from interfering with the merge---before its momentum carries it past the merge point.
\textcolor{black}{Note also that $\tilde{t}$ should also account for reasonable response times in vehicles following the CCAV. 
Indeed, the analysis in the examples below assumes the blocking behavior occurs at a rate sufficient to allow trailing vehicles (both AVs and HVs) to respond safely.} 

\subsection{In-ring reachability for mixed HVs and CCAVs: uniform case}
\label{subsec:in-ring-reachability-uniform-case}

Recognizing that the situation described in~\Cref{lemma_RLAV_worst} is adversarial, additional improvement in supervision scalability can be obtained with a typical (less adversarial) distribution on $A_j$ and $H_i$.

In a typical mixed autonomy case, the order of HVs and CCAVs relative to the merge point is interspersed.
As the CCAVs are supervision-aware and fully cooperative, the CCAV closest to the merge point may stop to accommodate a merge, and thus any vehicle after that CCAV cannot trigger supervision. 

To model this, let us denote the order statistics $A_{(1)} \leq ... \leq A_{(S)}$.
Without loss of generality and for ease of exposition, consider the restriction to $c = 1$ and the shorthand notation $d_i \coloneqq d_i(v_{i, 0}, a_{i,max}, t)$.

We consider three different distribution schemes for the in-ring vehicles\textcolor{black}{, beginning with the uniform case described in~\Cref{subsec:in-ring-reachability-uniform-case} and continuing to two nonuniform induced traffic distributions described in Sections~\ref{subsec:in-ring-reachability-induced-case} and~\ref{subsec:in-ring-reachability-induced-case-2}.
Instances of these three example cases are shown in~\Cref{fig:three_rings}.
In the first case, HVs are equally likely to be at any point in the ring. The second and third cases consider instances in which HVs are more likely to follow closely behind a CCAV.}

\begin{theorem} \label{theorem:sa_case1}
 Assume CCAV locations are independent from each other (i.e., $A_{j} \perp\!\!\!\!\perp A_{k}$ for all $j, k \in \{1, ..., S\}$) and all $HV$s' and $CCAV$s' locations are independent $H_i \perp\!\!\!\!\perp A_j$. Also  assume each vehicle's location is distributed uniformly at random along the length of the ring, i.e. $A_j \sim U[0, 1]$ for all $j \in \{1, ..., S\}$ and $H_i \sim U[0, 1]$ for all $i \in \{1, ..., R-S\}$. Then the probability that an arbitrary point on the ring is reachable over time horizon $t$ by a vehicle is upper bounded as follows
\begin{equation}
    P(C_{\text{in}}) \leq  \frac{1-\left(1-d_i\right)^{S+1}}{S+1}
\end{equation}
\end{theorem}
\begin{proof}
Since $A_j \stackrel{i.i.d.}{\sim} U[0, 1]$ for all $j \in \{1, ..., S\}$, we have from order statistics $A_{(1)} = \min_j A_j \sim Beta(1, S)$, and the probability density function for $A_{(1)}$ is $f_{A_{(1)}}(a) = S(1-a)^{S-1}, a\in [0, 1]$, where $a$ denotes the distance of the closest CCAV from the merge point~\cite{grimmett2020}. 
As $H_i \sim U[0, 1]$, the probability density function for $H_i$ is $f_{H_i}(h) = 1, h\in [0, 1]$.
By independence, we have the joint probability density function $f_{A_{(1)}, H_i}(a, h) = f_{A_{(1)}}(a)f_{H_i}(h) = S(1-a)^{S-1}, a \in [0, 1], h \in [0, 1]$ (and 0 otherwise).

We can therefore write:

\begin{equation}
    \begin{split}
        \{\text{HV $i$ } & \text{triggers sup. cond.}\} \\
        &= \{\text{HV $i$ triggers sup. cond.} \cap H_i \leq A_{(1)}\} \\
        &\qquad \cup \{\text{HV $i$ triggers sup. cond.} \cap H_i >A_{(1)}\}
    \end{split}
\end{equation}

where $P(\{\text{HV $i$ triggers supervision condition} \cap H_i >A_{(1)}\}) = 0$ in the current setting because any HV further away from the merge point than the nearest CCAV will be blocked by that CCAV from triggering the supervisor. 
Therefore, 
\begin{equation}
    \begin{split}
        \{\text{HV $i$ triggers sup. cond.} &\cap H_i \leq A_{(1)}\} \\
        & = \{H_i \leq d_i \cap H_i \leq A_{(1)}\}.
    \end{split}
\end{equation}

One can compute $P(\{H_i \leq d_i \cap H_i \leq A_{(1)}\})$ as

\begin{equation} \label{eq:absolute_prob}
    \begin{split}
        &\hspace*{-40pt}P(\{H_i \leq d_i \cap H_i \leq A_{(1)}\})\\
        &= \int_{a=0}^{d_i}\int_{h=0}^{a} f_{A_{(1)}, H_i}(a, h) dhda \\
        &\qquad + \int_{d_i}^{1}\int_{0}^{d_i} f_{A_{(1)}, H_i}(a, h)dhda\\\
        &= \int_{0}^{d_i}\int_{0}^{a} S(1-a)^{S-1} dhda \\
        &\qquad + \int_{d_i}^{1}\int_{0}^{d_i}S(1-a)^{S-1}dhda \\
        &= \frac{1-\left(1-d_i\right)^{S+1}}{S+1}
    \end{split}
\end{equation}
\end{proof}
A visual representation of this double integration for three examples is provided in~\Cref{fig:double_integration_uni_grad}.
In both the uniform distribution case and the nonuniform `platoon' distribution case (in which a number of HVs follow closely behind a CCAV) \textit{one benefit} of CCAVs is represented by the shaded triangle.
It corresponds to the portion of the ring at which in-ring HVs' reachability zones include the merge point, but which are blocked by an in-ring CCAV, and thus do not pose a danger to the merging vehicle. 
\textit{An additional benefit} may come in the form of the distribution shift of in-ring vehicles that the CCAV can cause, illustrated in the image with the color gradients, in~\Cref{fig:double_integration_uni_grad}b.
Here, not only does the CCAV block certain HVs from threatening the merge, but it also shifts the in-ring vehicle distribution such that an HV is less likely to be in the vicinity in the first place.
This case will be analyzed further in Example~\ref{sec:sa_case2}.

\textcolor{black}{\Cref{fig:three_rings} can aid in interpreting~\Cref{fig:double_integration_uni_grad}.}
\textcolor{black}{Both images} can be interpreted geometrically by imagining the colors as representing a third dimension on the image (as if rising out of the page towards the reader).
Taking the red to represent a greater volume of probability mass, we see that the shaded triangle \textcolor{black}{in~\Cref{fig:double_integration_uni_grad}} would also have more probability mass, and thus translate into greater supervision scaling than in the uniform distribution case.
At the same time, the distribution in~\Cref{fig:double_integration_uni_grad}(b) includes an adversarial element: the HVs between the CCAV nearest the merge point and the merge point itself are more likely to be distributed closely to the merge point, which increases the odds that they trigger the in-ring reachability condition.

Previously (in the NCAV case), $P({H_i \leq d_i}) = d_i$, so the absolute improvement in each term inside the union bound is 

\begin{equation} \label{eq:absolute_improvement}
    \begin{split}
        P({H_i \leq d_i}) - P(\{H_i \leq d_i &\cap H_i \leq A_{(1)}\}) \\
        &= d_i - \frac{1-\left(1-d_i\right)^{S+1}}{S+1}
    \end{split}
\end{equation}

and \textbf{the relative improvement of each term inside the union bound is} 

\begin{equation} \label{eq:relative_improvement}
    \begin{split}
        \frac{P({H_i \leq d_i}) - P(\{H_i \leq d_i \cap H_i \leq A_{(1)}\})}{P({H_i \leq d_i})} \\
        &\hspace*{-50pt}= \frac{d_i - \frac{1-\left(1-d_i\right)^{S+1}}{S+1}}{d_i} \\
        &\hspace*{-50pt}= 1 - \frac{1-\left(1-d_i\right)^{S+1}}{d_i(S+1)}
    \end{split}
\end{equation}

where $d_i \in [0, 1]$ as we normalize $c=1$.
This monotonically increases in $d_i$ and $S$---that is, the relative improvement increases with larger reachability zones or more CCAVs. 

\begin{example}
\textbf{(Uniform vehicle distribution)}
\end{example}
As a numerical example, when $d_i=0.1$ and $S=5$, the absolute probability (given via~\Cref{eq:absolute_prob}) is 0.078.
The previous absolute probability was 0.1, and so the absolute improvement (from~\Cref{eq:absolute_improvement}) is 0.022, and the relative improvement (from~\Cref{eq:relative_improvement}) is 0.219, or 21.9\%.

Interpreting this in terms of the reachability bound
\begin{equation}
    \begin{split}
        P(C_{\text{in}}) &\leq \sum\limits_{i=1}^{S-A}\{\text{HV $i$ triggers sup. cond.}\} \\
        &= \sum\limits_{i=1}^{S-A} P(\{H_i \leq d_i \cap H_i \leq A_{(1)}\}),
    \end{split}
\end{equation}
each term inside the sum is roughly 21.9\% less than it was previously.
Thus, this effectively drops 21.9\% of terms from the set $\{1, ..., S-A\}$ that the sum is taken over, and hence produces an additional 21.9\% improvement gain relative to the upper bound.

\subsection{In-ring reachability for mixed HVs and CCAVs: induced traffic cases}
\label{subsec:in-ring-reachability-induced-case}

\begin{example} \textbf{(An induced traffic vehicle distribution)} \label{sec:sa_case2}\end{example}
As the CCAV nearest the merge point $X$ may slow down or stop the HVs behind it, consider the scenario in which HVs may be more likely to be behind---and close to---the first CCAV, recalling~\Cref{fig:double_integration_uni_grad}b.
For example, imagine a CCAV leading a platoon of HVs.

Accordingly, we model the `wrapped-around' asymmetric distance between the HV following a CCAV using a truncated $Exp(1)$ distribution with support [0,1].\footnote{Specifically, we define $X = \begin{cases} h - a &\text{ if } h > a\\ 1 - (a-h) &\text { if }0 \leq h \leq a\end{cases}$ and we assume $X \sim Exp(1)$.}
As such, the probability density function of $H_i$ given $A_{(1)}$ can be written as:
\begin{equation}
    \begin{split}
        f_{H_i|A_{(1)}}(h|a) = \begin{cases}\frac{exp(-(h-a))}{1-exp(-1)} \;\;\text{if } a \leq h \leq 1\\\frac{exp(-(h+1-a))}{1-exp(-1)} \;\;\text{if } 0 \leq h < a\end{cases}.
    \end{split}
\end{equation}

Note that the $(h + 1 - a)$ term occurs when $h < a$, and therefore corresponds to the unlikely case when an HV is closer to the merge point than a CCAV. 
In this setting, this adjustment is necessary to ensure that when $h = 0$ the probability produced is equal to the $a \leq h$ case when $h = 1$.
For a visual explanation see \Cref{fig:double_integration_uni_grad}b: this adjustment is what ensures the probability distribution smoothly 'wraps around' from the top of the diagram to the bottom.

\begin{figure*}[t!]
    \centerline{
        \includegraphics[width=0.33\textwidth]{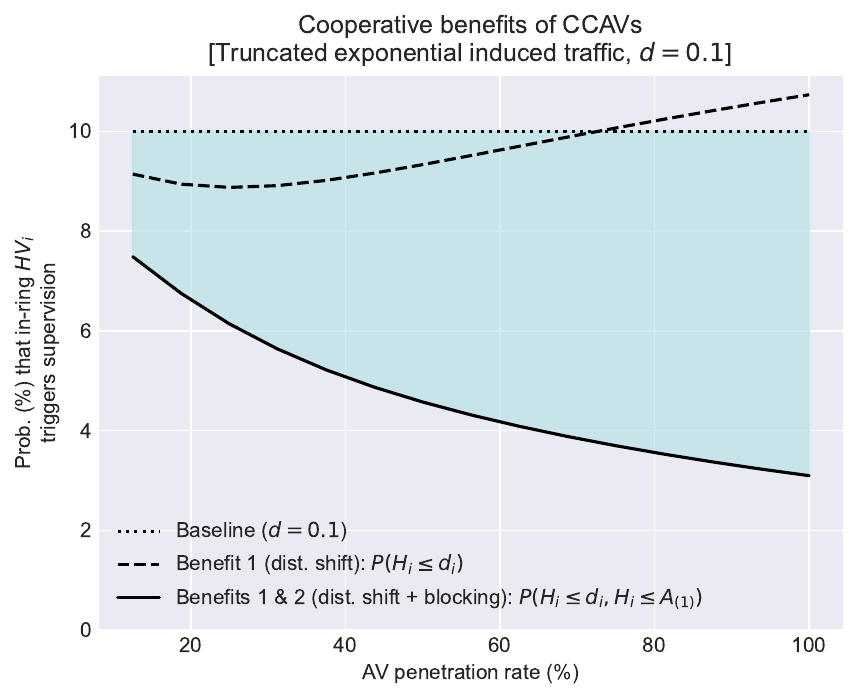}
        \includegraphics[width=0.34\textwidth]{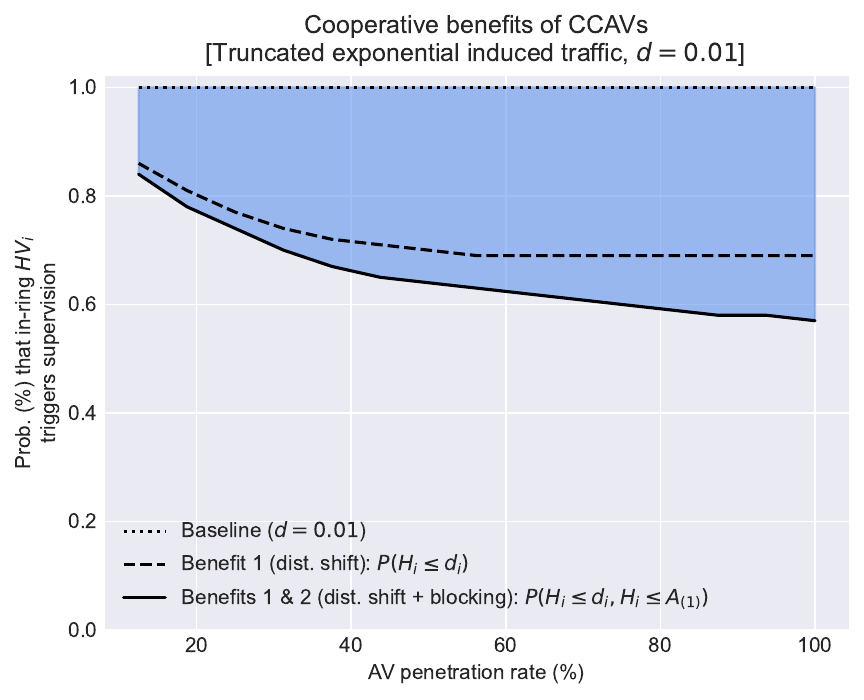}
        \includegraphics[width=0.32\textwidth]{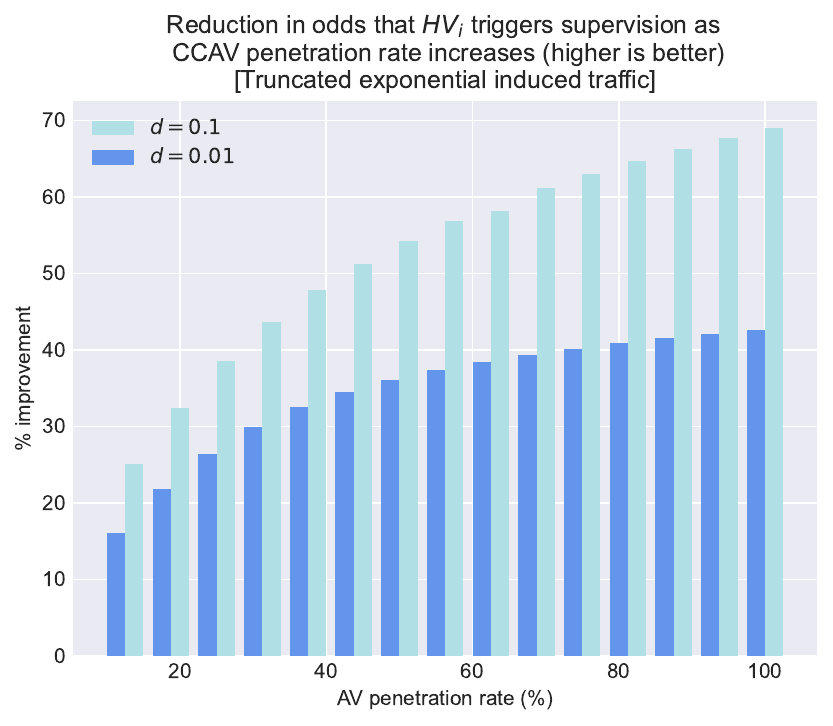}
    }
    \caption{Visualization of the \textcolor{black}{cooperation-based} benefits of CCAVs for the situation described in~\Cref{sec:sa_case2}. The left two plots show the probability that the merge point is within the reachable zone of $HV_i$ in the \textcolor{black}{UCAV and NCAV} cases (dotted line) and in the case in which CCAVs are present (solid line, $P(H_i \leq d_i, H_i \leq A_{(1)})$) for two different $d$ values. \textcolor{black}{The dashed line removes the blocking effect for comparison.} The CCAVs' relative improvement over the \textcolor{black}{UCAV and NCAV} cases is shaded in the two left plots and explicitly represented with corresponding colors in the bar chart. \textcolor{black}{Note these plots consider a single hypothetical HV---not the system's overall probability of triggering supervision.}
    See~\Cref{appsec:tables} for the source data. 
    }
    \label{fig:truncatedexpdist_results}
\end{figure*}

\begin{figure*}[t!]
    \centerline{
        \includegraphics[width=0.33\textwidth]{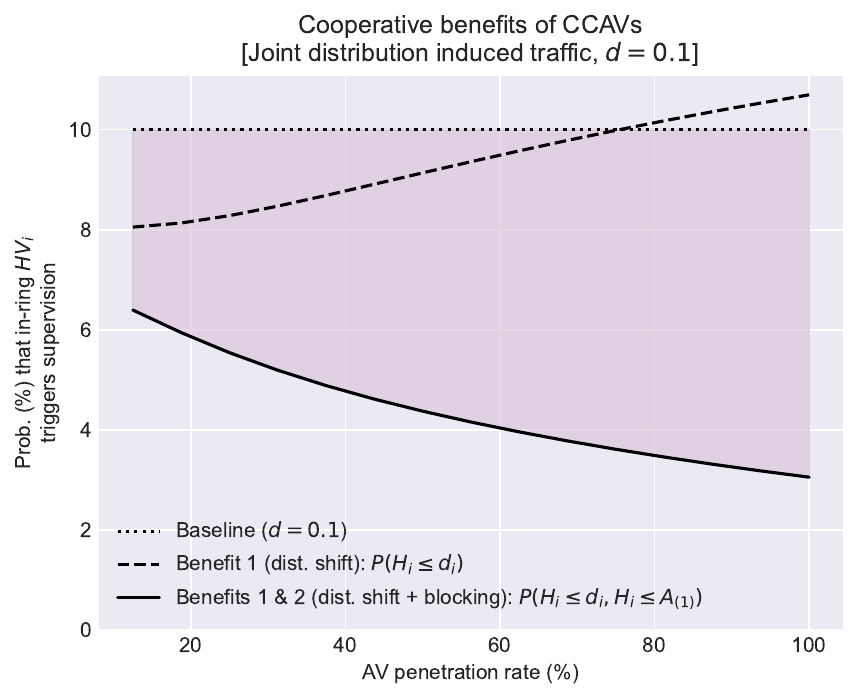}
        \includegraphics[width=0.34\textwidth]{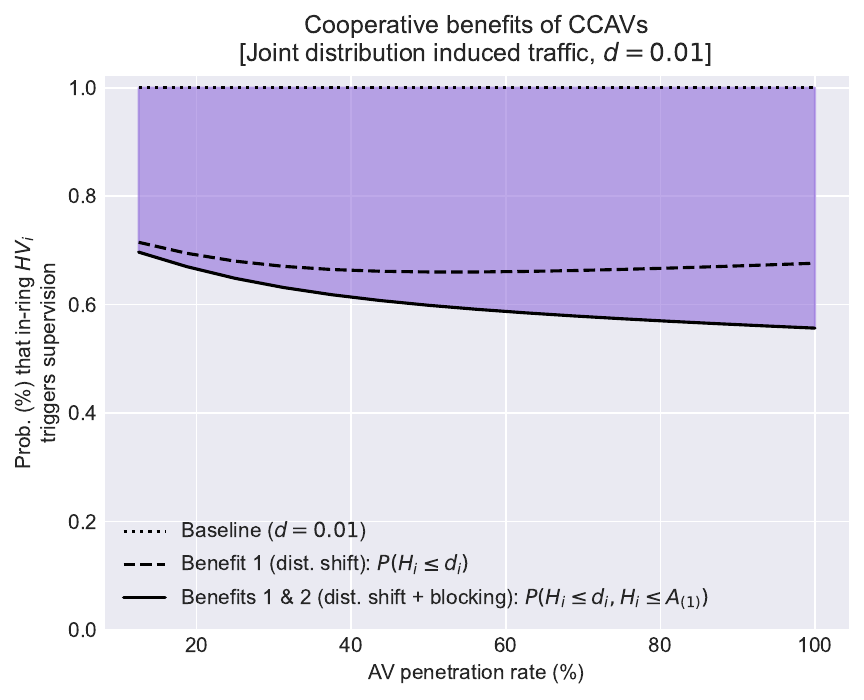}
        \includegraphics[width=0.32\textwidth]{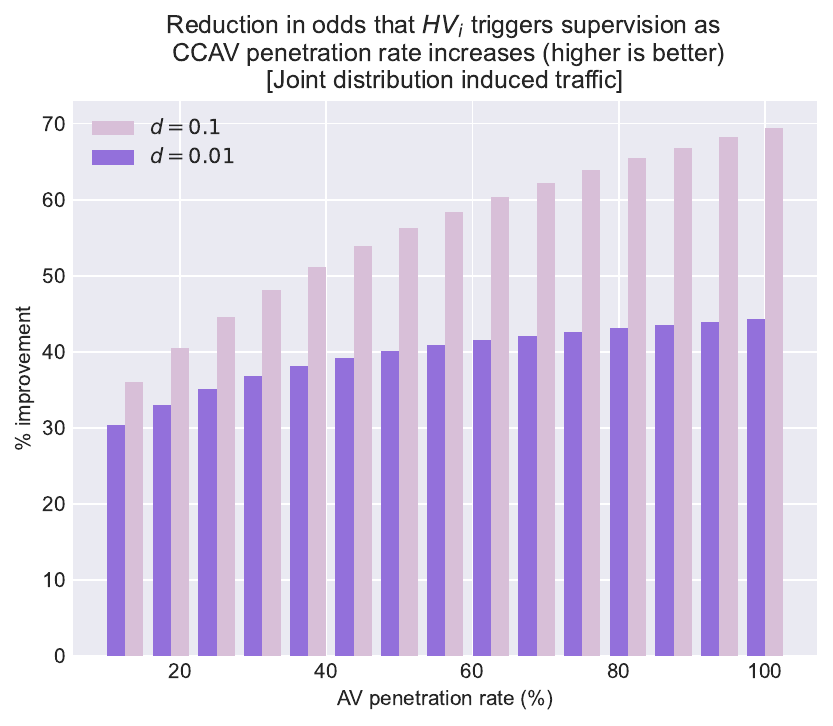}
    }
    \caption{Visualization of the \textcolor{black}{cooperation-based} benefits of CCAVs for the situation described in~\Cref{sec:sa_case3}, similar in format to the above figure. The dashed line represents the gains (relative to the dotted line) that the in-ring distribution shift incurred by CCAVs produces. The solid line compounds that with the benefit from their ability to sometimes block HVs that would otherwise trigger supervision. See~\Cref{appsec:tables}'s tables for the source data.  
    }
    \label{fig:jointdist_results}
\end{figure*}

For tractable computation assume $f_{A(1)}(a) = \frac{Sexp(-Sa)}{1-exp(-S)}$ to be the truncated $Exp(S)$ random variable on [0,1].
The resulting joint distribution is

\begin{equation}
    \begin{split}
        f_{A_{(1)}, H_i}(a, h) &= f_{H_i|A_{(1)}}(h|a)f_{A_{(1)}}(a) \\
        &= \begin{cases}\frac{exp(-(h-a))}{1-exp(-1)}\frac{Sexp(-Sa)}{1-exp(-S)} \;\;\text{if } a \leq h \leq 1\\\frac{exp(-(h+1-a))}{1-exp(-1)}\frac{Sexp(-Sa)}{1-exp(-S)} \;\;\text{if } 0 \leq h < a\end{cases}.
    \end{split}
\end{equation}

We can compute the probabilities of interest using the same integration as in Case 1, but with a different joint probability density function. Closed-form probabilities can be obtained as
\begin{align}
    P(\{H_i \leq d_i \cap H_i &\leq A_{(1)}\}) \nonumber\\
       &= \frac{e^{(S - d_i S)} (e^{d_i S} -e^{d_i}S+ S-1)}{(e-1) (e^S-1) (S-1)} \\
       &\qquad + \frac{(e^d_i-1)S (e^{d_i + S} - e^{d_i S+1})}{(e-1) (e^S-1) (S-1)}\nonumber\\
    P(\{H_i \leq d_i \cap H_i &> A_{(1)}\}) \nonumber\\
        &= \frac{e^{S+1}( e^{-d_i S} - e^{-d_i}S + S - 1)}{(e-1)(e^S-1)(S-1)}\\
    P(\{H_i \leq d_i\})\hspace{13pt}&\nonumber\\
        &= P(\{H_i \leq d_i \cap H_i \leq A_{(1)}\}) \\
        &\qquad + P(\{H_i \leq d_i \cap H_i > A_{(1)}\})\nonumber
\end{align}

As a numerical example, with $d_i=0.1$, we obtain a \raisebox{0.5ex}{\texttildelow}44\% improvement from the previous upper bound when $S=5$ and a \raisebox{0.5ex}{\texttildelow}68\% improvement when $S=15$.

See~\Cref{fig:truncatedexpdist_results} for the numerical results with $d_i = 0.1$ and $d_i = 0.01$ respectively in a setting with 16 in-ring vehicles total. 
\textcolor{black}{Although no in-ring HV would actually be present in the 100\% AV penetration setting, values are reported for illustrative purposes.}
The data used for the plots is included in~\Cref{tab:numerical_probabilities21} and~\Cref{tab:numerical_probabilities22} in the appendix.
In reality, the true distribution is likely to lie somewhere between the two cases outlined in this section. 
Note the relative improvement column in the tables is computed by comparing the $P(H_i \leq d_i, H_i \leq A_{(1)})$ column as AV penetration rate increases with $P(H_i \leq d_i) = d_i$, which corresponds to the case in \textcolor{black}{which the probability of a given HV's reachability zone intersecting the merge point is unaffected by the CCAVs.}
\textcolor{black}{To illustrate the relative impact of the CCAVs' blocking benefit relative to their distribution shift benefit, we ablate the effect of blocking.} 

\begin{example} \textbf{(A more realistic CCAV case)} \label{sec:sa_case3}\end{example}
\subsection{Supervision requirements for a single rotary}
\label{subsec:in-ring-reachability-induced-case-2}
\begin{figure*}[tbhp]
    \centering
    \includegraphics[width=0.9\textwidth]{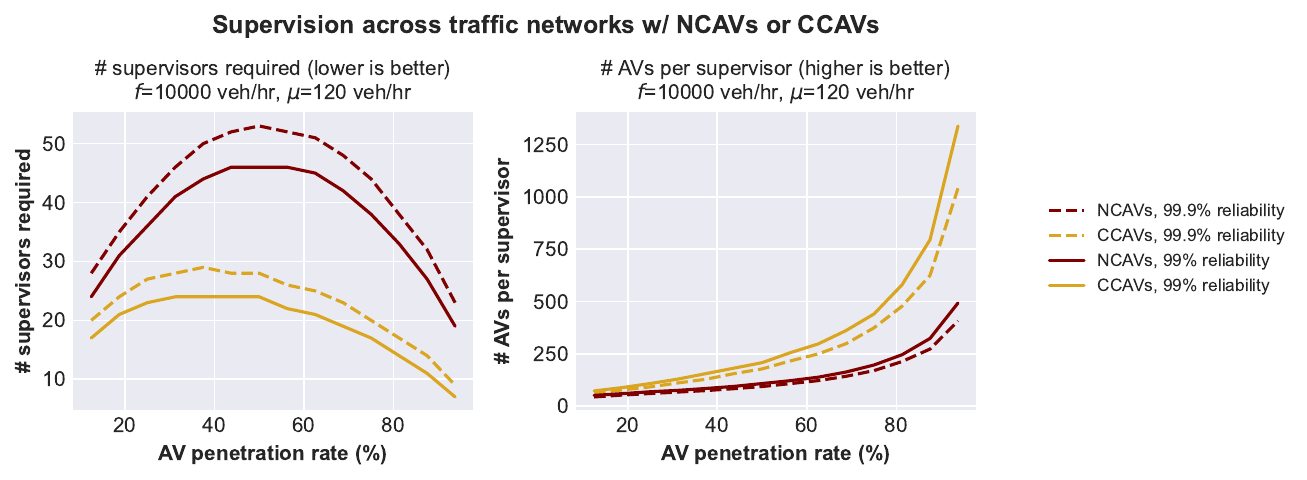}
    \caption{With connected AVs the number of supervisors necessary to achieve desired high \textcolor{black}{reliability} thresholds no longer increases monotonically with rising AV penetration rates, as occurred in the UCAV (baseline) case shown in~\Cref{fig:geneqn_num_sup_num_avs_single}. The CCAVs substantially outperform the NCAVs. In fact, the CCAV case exhibits an overall trend inversely related with increasing AV penetration. Increasing AV presence boosts the number of AVs requiring supervision at merge time, but also enables AVs to have greater supervision-reducing effects within the ring.}
    \label{fig:network_req_sup_and_avs_single}
\end{figure*}

\begin{figure*}[tbhp]
    \includegraphics[width=0.9\textwidth]{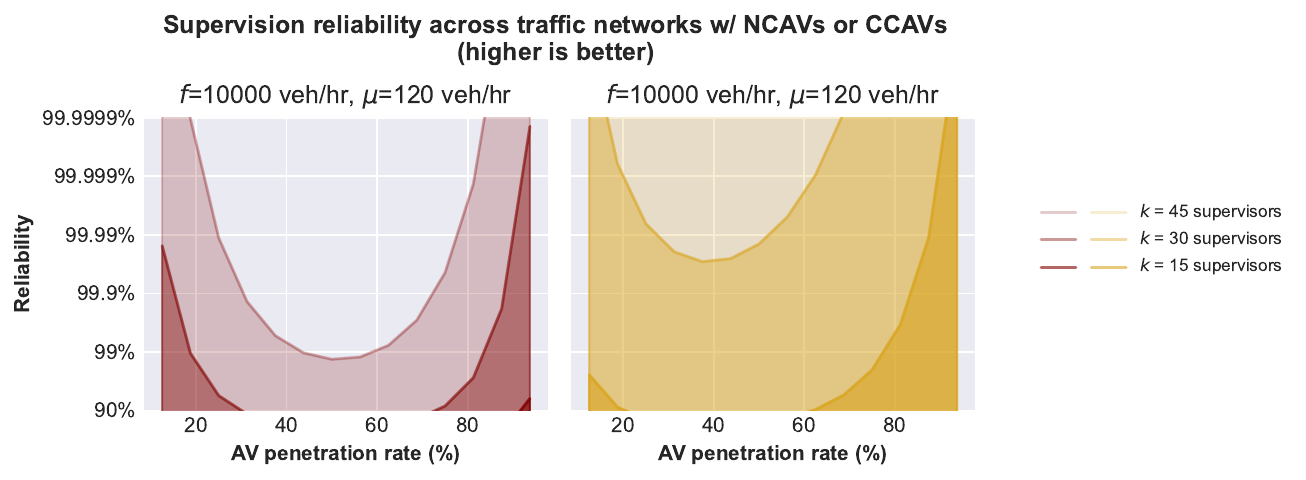}
    \caption{A visualization of the \textcolor{black}{reliability} thresholds NCAVs (left) and CCAVs (right) can achieve at a given task arrival and service rate. \textcolor{black}{Other rates exhibit similar trends; see~\Cref{subsec:traffic_network_arrays}.} 
    Comparing these plots---while making special note of the logarithmic y-axis---illustrates how CCAVs enable a fixed number of supervisors to achieve a substantially higher \textcolor{black}{reliability} threshold.
    }
    \label{fig:network_conf_both_single}
\end{figure*}

In the above example we assume both the HVs behind and in front of the first CCAV follow a truncated exponential distribution. Such a distribution is slightly unrealistic (and slightly adversarial) for HVs in front of the first CCAV because it implicitly says that those HVs have a greater probability of being closer to the merge point than far away (that is, closer to the first CCAV). Recall that we define the asymmetric distance following the traffic flow direction. A less adversarial distribution would say the HVs before the first CCAV follow a uniform distribution between that CCAV and the merge point. In fact, here we consider such a joint distribution for HVs and CCAVs, where the conditional probability density function of $H_i$ given $A_{(1)}$ is

\begin{equation}
    \begin{split}
        f_{H_i|A{(1)}}(h|a) = \begin{cases}\frac{exp(-(h-a))}{1-exp(-1)} \;\;\text{if } a \leq h \leq 1\\\frac{e^a-1}{a(e-1)} \;\;\text{if } 0 \leq h < a\end{cases}
    \end{split}
\end{equation}
where $\frac{e^a-1}{a(e-1)}$ is set so that the conditional probability $f_{H_i|A{(1)}}(h|a)$ integrates to 1 for each $a$. That is, this is a truncated $Exp(1)$ random variable on [0,1] where the random variable is the asymmetric distance from an HV to a CCAV (following the traffic flow) if the HV is behind the first CCAV, but conditionally uniform if the HV is before the first CCAV. 
A visual representation of this conditional probability density function can be seen in~\Cref{fig:double_integration_uni_grad}(c).

Again, for tractable computation assume $f_{A_{(1)}}(a) = \frac{Sexp(-Sa)}{1-exp(-S)}$ to be the truncated $Exp(S)$ random variable on [0,1].
The resulting joint distribution is

\begin{equation}
    \begin{split}
        f_{A_{(1)}, H_i}(a, h) &= f_{H_i|A_{(1)}}(h|a)f_{A_{(1)}}(a) \\
        &= \begin{cases}\frac{exp(-(h-a))}{1-exp(-1)}\frac{Sexp(-Sa)}{1-exp(-S)} \;\;\text{if } a \leq h \leq 1\\\frac{e^a-1}{a(e-1)}\frac{Sexp(-Sa)}{1-exp(-S)} \;\;\text{if } 0 \leq h < a\end{cases}.
    \end{split}
\end{equation}

The results of this computation are shown in \Cref{fig:jointdist_results} and the corresponding tables in the appendix. 
~\Cref{tab:numerical_probabilities31} contains the result when fixing $d_i=0.1$.
Table~\ref{tab:numerical_probabilities32} contains the corresponding result when fixing $d_i = 0.01$.

We can see that the relative improvements are larger than those in Example~\ref{sec:sa_case2}, as the distribution of HVs before the first CCAV is less adversarial.

We can leverage this work detailing the distribution of in-ring vehicles---and more specifically, the $P(H_i \leq d_i, H_i \leq A_{(1)})$ values, from which we construct $P(C_{\text{in}})$ and therefore $\lambda$---with~\Cref{eq:pm} to get more direct insight into quantities and trends of interest.

\subsection{Supervision requirements for a traffic network}
\label{subsec:sup_req_for_traffic_network}
The analysis to this point has been centered on a single ring road, and the results for that setting are included in the Appendix. 
However, returning to our objective of asking, ``How many supervisors are necessary?", particularly in the setting in which a team of human supervisors remotely supervise a fleet of AVs, we are interested in a broader traffic network. 

Using our kinematics-based reachability distance and the on-ramp density for interstate highways from~\cite{pilko2007safety}, we find that reasonable values for $d$ range from 0.08 to 0.11 in Washington and California interstate systems, and thus will proceed with $d=0.1$ for this section of the analysis.
However, we observe that the precise value of $d$ to use will vary according to context and desired analysis; still, the trends outlined below persisted for a range of values. 
The CCAV $P(C_{\text{in}})$ values are thus drawn from~\Cref{tab:numerical_probabilities31}.
On-ramp flow rates $f$ can range from 0 to $\sim$1,900 veh/hr before single-lane ramps become saturated, and the mean time for a supervisor to monitor a merge is estimated to be between 15 and 60 seconds, including the time to safely gain situational awareness and assume effective control of the vehicle remotely from the autonomous driver~\cite{daw2019beyond, hess1963capacities}.

\textcolor{black}{Thus, we set $f$ = [7,000, 10,000, and 17,000] for the plots in the Appendix and highlight $f$ = 10,000 for our initial plots. This value can equivalently refer to a network with 10 on-ramps each producing 1,000 veh/hr, or 100 on-ramps each producing 100 veh/hr flowing onto a free-flowing interstate.}

\textcolor{black}{
Results are presented in Figures~\ref{fig:network_req_sup_and_avs_single} and~\ref{fig:network_conf_both_single}.
Additional results for other inflow and service rates (demonstrating trend generality) are shown in the Appendix in Figures~\ref{fig:network_req_sup_array}, \ref{fig:network_num_avs_array}, \ref{fig:network_conf_noncoop_array}, and~\ref{fig:network_conf_coop_array}.
}
We see that CCAVs enable supervision time that is on the whole inversely related with AV penetration, and that the number of AVs each supervisor can monitor at a given reliability threshold monotonically increases along with the AV penetration rate. 
The plots also illustrate the extent to which cooperation among the connected AVs boosts \textcolor{black}{reliability} in supervision.
In particular, note the logarithmic scale on the y-axis in these plots. 
The `\textcolor{black}{reliability}' metric is calculated as $1 - P_k$. This means that at a \textcolor{black}{reliability} threshold of 99.9\%, 1 in 1,000 vehicles---on average---for which the reachability conditions require supervision would not be able to receive supervision (because there is not a supervisor available to immediately provide it). 
Observe that this does not necessarily mean a crash would occur, given the conservative nature of the reachability constraints. Future work via microsimulation could gather data for analysis of this. 

\textcolor{black}{Of particular note is the improvement of the CCAV case over the NCAV case.
~\Cref{fig:highlight_example} illustrates how }\textcolor{black}{CCAVs can provide orders of magnitude more \textcolor{black}{reliability} without additional human supervisors (note the log scale).}

The shape of the various curves can be explained by three interacting effects as the AV penetration rate increases. 
\textcolor{black}{See~\Cref{fig:interacting_effects}.}
Exerting upward pressure on the number of supervisors required to maintain a given \textcolor{black}{reliability} level as AV penetration increases is the growing number of merging AVs that present potential supervision needs.
Exerting downward pressure for both the noncooperative and cooperative connected AV cases is the fact that as AV penetration increases, a greater proportion of vehicles have the truncated reachability zones that connectivity grants, reducing the average area of the highway that is covered by reachability zones. 
Finally, for the CCAVs, there exist the benefits discussed in the presentation of the other plots: (a) the ability to shift the in-ring vehicles' collective distribution such that each vehicle's reachable zone is more likely to overlap with that of another vehicle, thus leaving more of the ring open to accommodate merges, and (b) the ability of CCAVs to block would-be dangerous HVs from encroaching on merging AVs as they join the highway flow. 
\textcolor{black}{This final effect contributes both to the goal of reducing the number of necessary human interventions and reducing the time required for those that must occur.
The truncated reachability effect contributes to the goal of reducing the number of human interventions.} 
The curves presented in all four figures can be understood in these terms. 

\textcolor{black}{We observe that the `dip' in the reliability curves consistently occurs at a similar point, indicating the mixed autonomy cases with relatively equal proportions of human and autonomous vehicles provide the most challenging supervision scenarios.
Also notable is that the dip shifts leftward in the CCAV setting relative to the NCAV setting.
This phenomenon is particularly visible in the reliability plots in the Appendix, such as Figures~\ref{fig:network_conf_noncoop_array} and~\ref{fig:network_conf_coop_array}.
This may have to do with the nonlinear influence of vehicles on others in the system and is an open avenue for future study.}

While the values here are representative of a network larger than a single rotary, the plots in the Appendix illustrate that similar qualitative behavior occurs even on the much smaller case of a single rotary with four on- and off-ramps. 

\begin{figure} [t]
    \centering
    \includegraphics[width=0.45\textwidth]{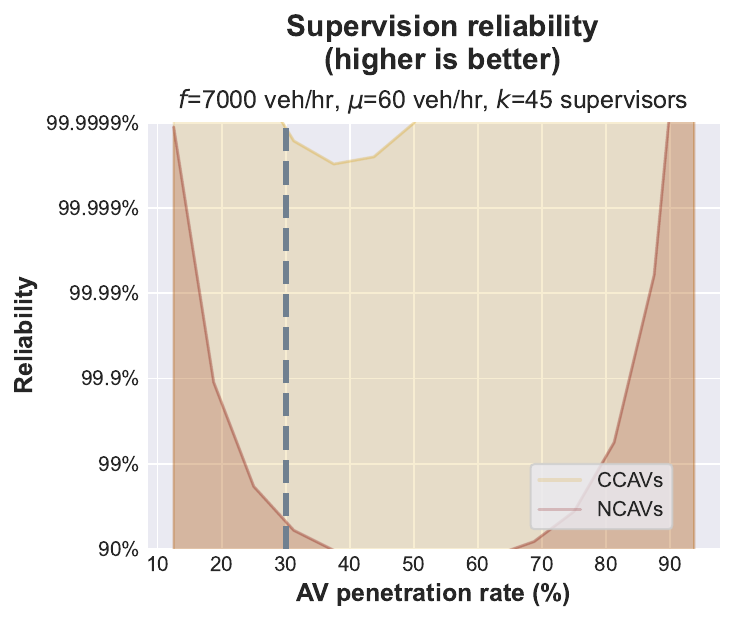}
    \caption{\textcolor{black}{Under connected AV cooperation, the operator team successfully covers 99.9999\% of cases when the AV penetration rate is 30\%. In contrast, without connected AV cooperation, the same team covers \textcolor{black}{just over} 90\% of cases.} 
    }
    \label{fig:highlight_example}
\end{figure}

\begin{figure*} [t]
    \centering
    \includegraphics[width=1.0\textwidth]{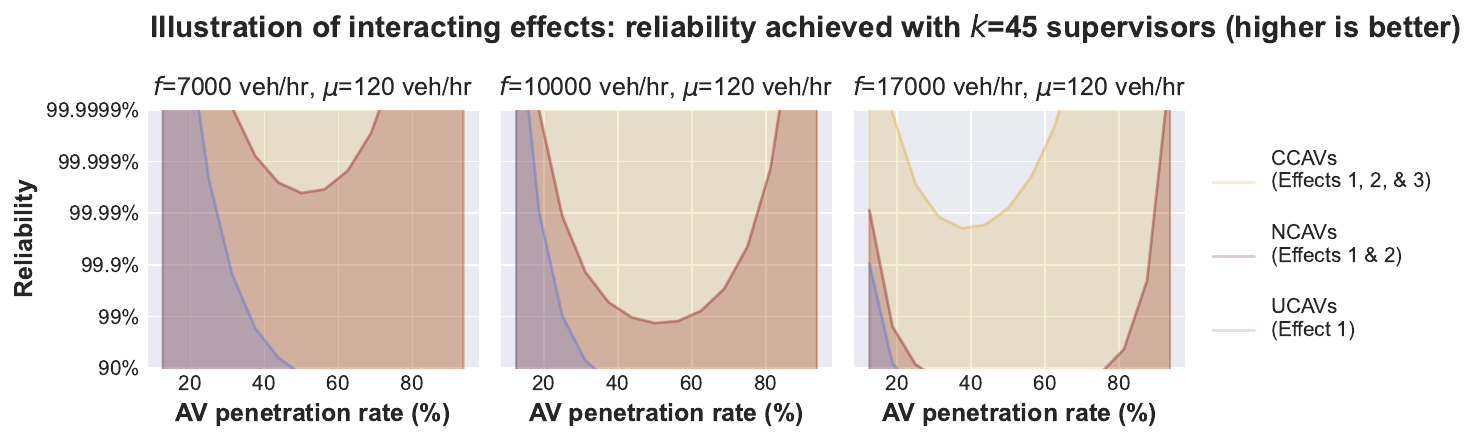}
    \caption{\textcolor{black}{As the AV penetration rate increases, three notable factors affect the system’s supervisability: (1) more AVs to supervise, (2) a greater share of vehicles with truncated reachability zones, and (3) more AVs to assist merges. (1) makes supervision more difficult. (2) and (3) ease supervision. The unconnected AV curve (in blue) is affected only by 1, as the AVs neither assist merges nor have truncated reachability zones. The NCAV curve (in red) shows what happens with only 1 and 2, since NCAVs do not assist merges. The CCAV curve (in yellow) is affected by 1, 2, and 3. \textcolor{black}{Recall that Effect 3 is the composite of the two benefits described in Figures~\ref{fig:truncatedexpdist_results} and~\ref{fig:jointdist_results}.}} 
    }
    \label{fig:interacting_effects}
\end{figure*}

\section{Implications for Autonomous Driving}
\label{sec:problemscoping}

It is worth pausing for a moment to appreciate the broader context of these improvements and the leveraging effect that can be achieved in the traffic setting by adopting autonomy with supervision assistance for the driver. 
\textcolor{black}{The results suggest preventing the need for supervision is more powerful than having more supervisors.}
Not only does this provide further motivation for investigating scalable supervision of mixed autonomy systems in the specific context of AVs, but it also serves as a case study for determining the potential for supervision scaling to achieve practical gains in a given setting.
In so doing, this analysis provides a framework that could be used to evaluate other tasks for scalable supervision.

\subsection{Improved safety and user experience}
AVs provide a compelling case for investigating scalable supervision because short, particularly dangerous events---such as merging or navigating a construction zone---are interspersed within longer, easier events---such as cruising a highway. 
Imagine a future where \textcolor{black}{AVs can safely handle mundane driving} like occurs on an interstate highway.
An intervention that prevents a human driver from lifting a finger during a five-seconds-long merge may thus not only spare the person interruption for those five seconds, but indeed may have in effect spared her interruption over 20 minutes if the merge interlinked 10-minute stretches on two separate highways. 

Certainly, unforeseen events can occur at any point in the driving process, but that is beyond the scope of this research, which focuses on the illustrative and pre-specified event of merging. 
Other instances which present increased risks or unpredictability but can be anticipated include construction zones, school zones, corridors particularly congested with foot traffic, and known traffic accident hotspots~\cite{erdogan2008geographical}. 

\subsection{Informed fleet operations}
Beyond the improved driving experience and safety for an individual, in the further future such scaling might be a crucial element enabling the widespread deployment of autonomous vehicles, and thus crucial for achieving the environmental and social benefits outlined previously. 
Consider a company that seeks to provide a fully hands-off driving experience. 
Even if the company's AVs are generally quite effective at navigating roadways, given the difficulty of some tasks like merging and consumers' desire for additional safety, the firm may look for additional safety measures. 
If it were sufficiently economical (and assuming low-latency connections, etc.), one way to do so might be to hire remote supervisors to monitor the AV fleet and provide real-time control when necessary. 
Thus, the question of interest for this company would not only be reductions in the number of interventions and the time required over some baseline, but how many remote supervisors need to be hired as a whole to supervise the entire fleet of AVs. 

One aim of this work is to provide rigorous methodology for informing such estimates by refining the estimated time required for supervision of the AVs. When tailored to the locale and services of an AV operator, the outcome of the proposed methodology could be incorporated into an estimate for fleet operations.

Based on some moderate assumptions, we might expect that---even if we were to require that a supervisor monitor \textit{each and every} AV merge, regardless of reachability conditions and without any of the scaling techniques used in this work---one of these remote supervisors might be able to supervise between 18 and 66 AVs for the task of highway merging in the United States. 
The logic behind this estimate is as follows. 
In the US approximately 33\% of total vehicle miles traveled (VMTs) occur on interstates, other freeways, or expressways, where ``access and egress points are limited primarily to on- and off-ramps''~\cite{mcguckin2018summary, FHA2018_VMT, usdot_2015status}.
Making the conservative assumptions that on these roads the average speed is 50mph and the average trip length is 15 miles, the total time for each trip is 0.3 hours.
If the time it takes for a supervisor to assume control and execute a merge task safely is 60 seconds, and assuming one merge per trip, then a supervisor is needed for 5.55\% of every trip. 
Taking the inverse of this, we see that---if tasks are allocated without gaps or delays---one supervisor could theoretically monitor merges for 18 highway vehicles, even assuming that every merge must be supervised. That is, this value is \textit{without} any of the scalable supervision benefits discussed in this work.
Since VMTs on interstates, other freeways, and expressways account for only a third of all VMTs, one supervisor operating in such a manner could supervise up to 52 vehicles (since the majority never make a highway merge).
Of course, the scaling potential is substantially greater once we allow for the supervision methods outlined previously.

\section{Conclusion}

This work explored \textcolor{black}{an operational reliability} perspective on improving safety in mixed autonomy settings by investigating human supervision of AVs in a simulated merging task.
We investigated the question: Can we do better than the present-day industry standard of persistent supervision? 
Our findings indicate that clever allocation of human resources for supervision tasks may ease near-term adoption of imperfect autonomous agents in safety-critical environments.
\textcolor{black}{Key conclusions include:
\begin{itemize}
    \item AV teaming can enable supervision that is both safe and scalable, \textcolor{black}{with orders-of-magnitude improvement in supervision requirements. This cooperation may enable autonomous systems' deployment at scale in certain settings.} 
    \item Queuing-theoretic analysis and order statistics allow for the establishment of high \textcolor{black}{reliability} upper bounds on human supervision requirements in the given setting. 
    \item Scalable supervision provides a method for achieving human-level safety even when autonomous system performance guarantees are elusive.
\end{itemize}} 
Future work could use traffic microsimulators to empirically investigate the distribution produced, and could also evaluate the impact of the supervision-aware AVs' cooperative behavior on traffic flows.
These simulations could extend to investigate how learning can be incorporated to exhibit behaviors similar to---or even superior to---those discussed in this work.
More generally, continued research in this direction could adopt alternative reachability analysis tools to investigate supervision scaling's potential in other mixed autonomy settings.

\section*{Acknowledgments}
We are grateful to Dr. Eric Horvitz for inspiring our exploration of AV teaming.
We would also like to thank Zhongxia Yan for feedback and assistance with network design and Jiaqi Zhang for discussion about and assistance with elements of analysis.
\textcolor{black}{This work was partially supported by the MIT Amazon Science Hub, MIT-IBM Watson AI Lab, MIT Energy Initiative (MITEI) Mobility Systems Center, MIT Mobility Initiative, and National Science Foundation (NSF) under grant number 2149548.}

{\appendix[]
\section*{Useful Statements}
We provide useful probability statements that we adopt in our proofs. The statements can be found in a standard probability textbook~\cite{grimmett2020}.

\subsection{Binomial distribution parameterized by a Poisson random variable} 

\begin{equation} \label{poisson_binomial_statement}
    \begin{split}
        \text{Let $N \sim Poisson(\lambda), X | N $} &\text{$\sim Binomial(N, p)$,} \\
        &\text{then $X \sim Poisson(\lambda p)$.}
    \end{split}
\end{equation}

\subsection{Poisson distribution of summed independent random variables} 

\begin{equation} \label{poisson_random_var_sum_statement}
    \begin{split}
        \text{Let $X \sim Poisson(\lambda), Y$ } &\text{$ \sim Poisson(\mu)$ and $X, Y$ } \\
        &\hspace*{30pt}\text{are independent.} \\
        &\hspace*{-105pt}\text{Then $X + Y \sim Poisson(\lambda + \mu)$.}
    \end{split}
\end{equation}

\subsection{Plots for understanding the relationships expressed in~\Cref{theorem_queue}}
\textcolor{black}{
See Figures~\ref{fig:geneqn_num_sup_array}, \ref{fig:geneqn_avs_per_sup_array}, and~\ref{fig:geneqn_confidence_array}. These illustrate similar trends for a variety of configurations.}

\begin{figure*}[tbhp]
    \centering
    \includegraphics[width=0.85\textwidth]{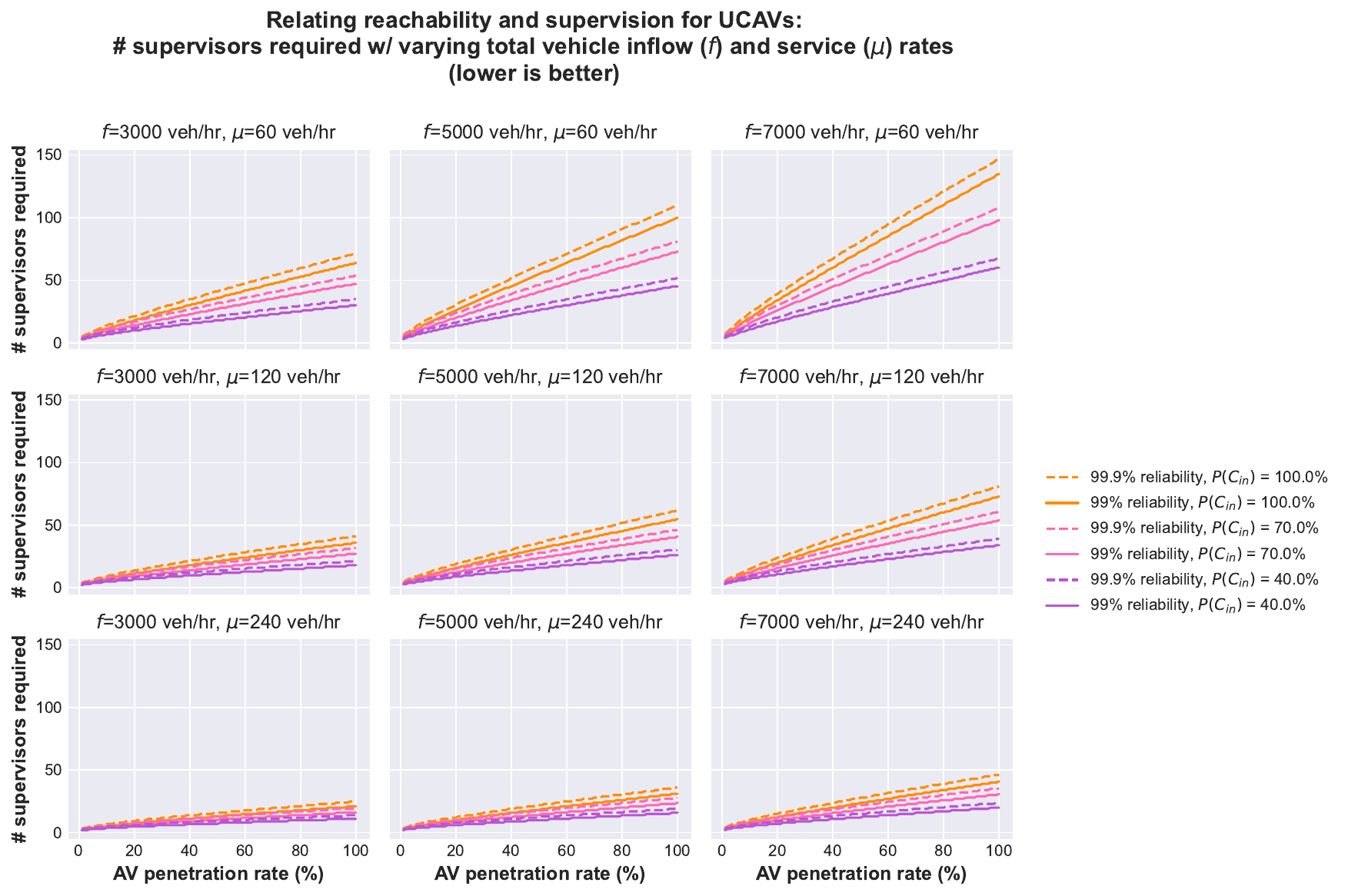}
    \caption{\textcolor{black}{Results described in~\Cref{fig:geneqn_num_sup_num_avs_single} (UCAV case) for a variety of total vehicle inflow rates $f$ and service rates $\mu$.
    Note the trends are similar across configurations.}
    The values for $f$ are feasible light-traffic on-ramp values in a traffic system and the $\mu$ values represent monitors that can supervise a merge in an average of 60, 30, or 15 seconds, respectively. Further details on parameter selection are included in the main text of the paper.
    } 
    \label{fig:geneqn_num_sup_array}
\end{figure*}

\begin{figure*}[tbhp!]
    \centering
    \includegraphics[width=0.85\textwidth]{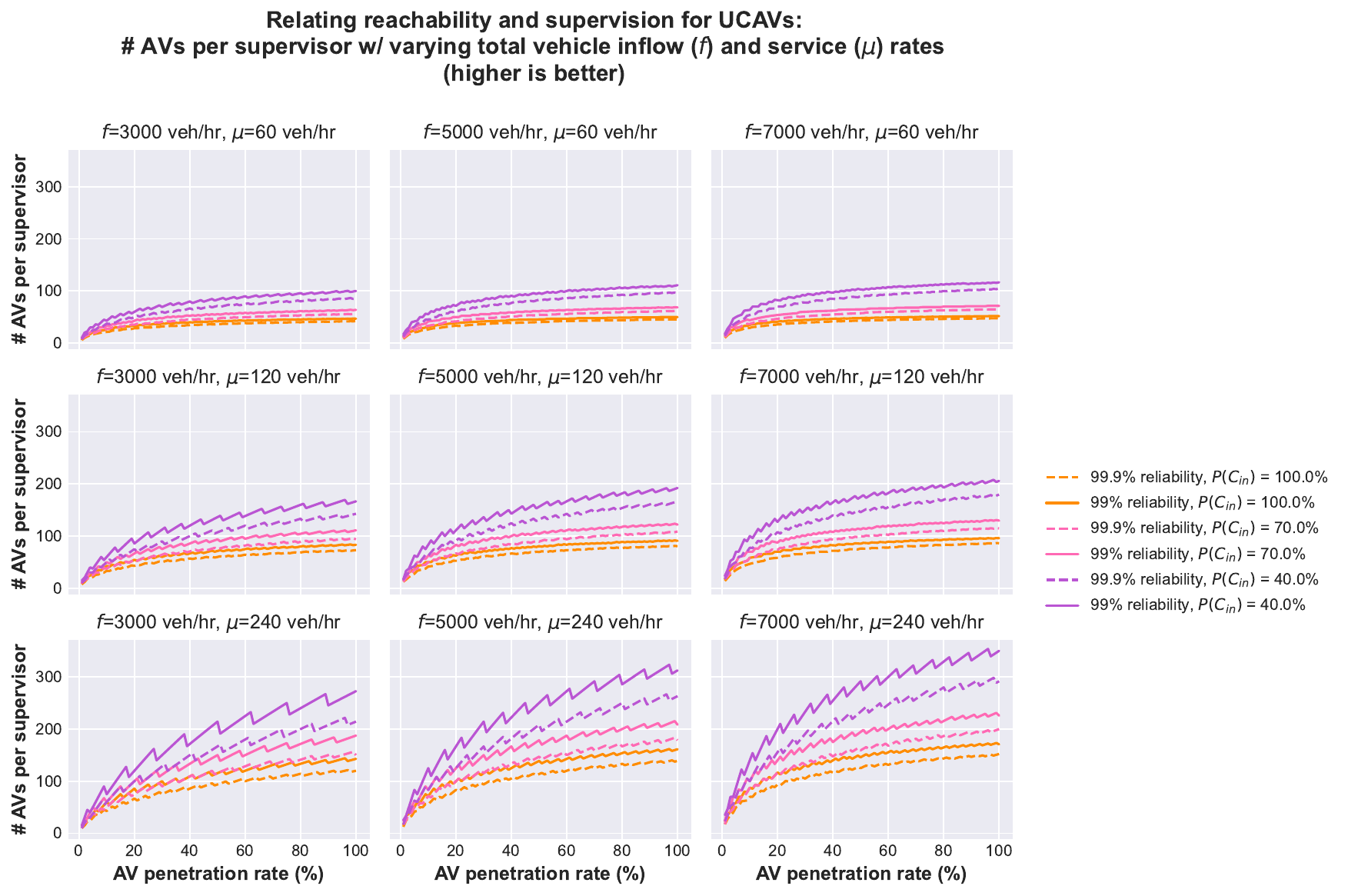}
    \caption{An illustration of how the number of AVs per supervisor scales with increasing AV saturation for the `baseline' case in which there is no merge assistance from noncooperative or cooperative connected AVs. \textcolor{black}{Trends are similar across configurations.}}
    \label{fig:geneqn_avs_per_sup_array}
\end{figure*}

\begin{figure*}[tbhp!]
    \centering
    \includegraphics[width=0.85\textwidth]{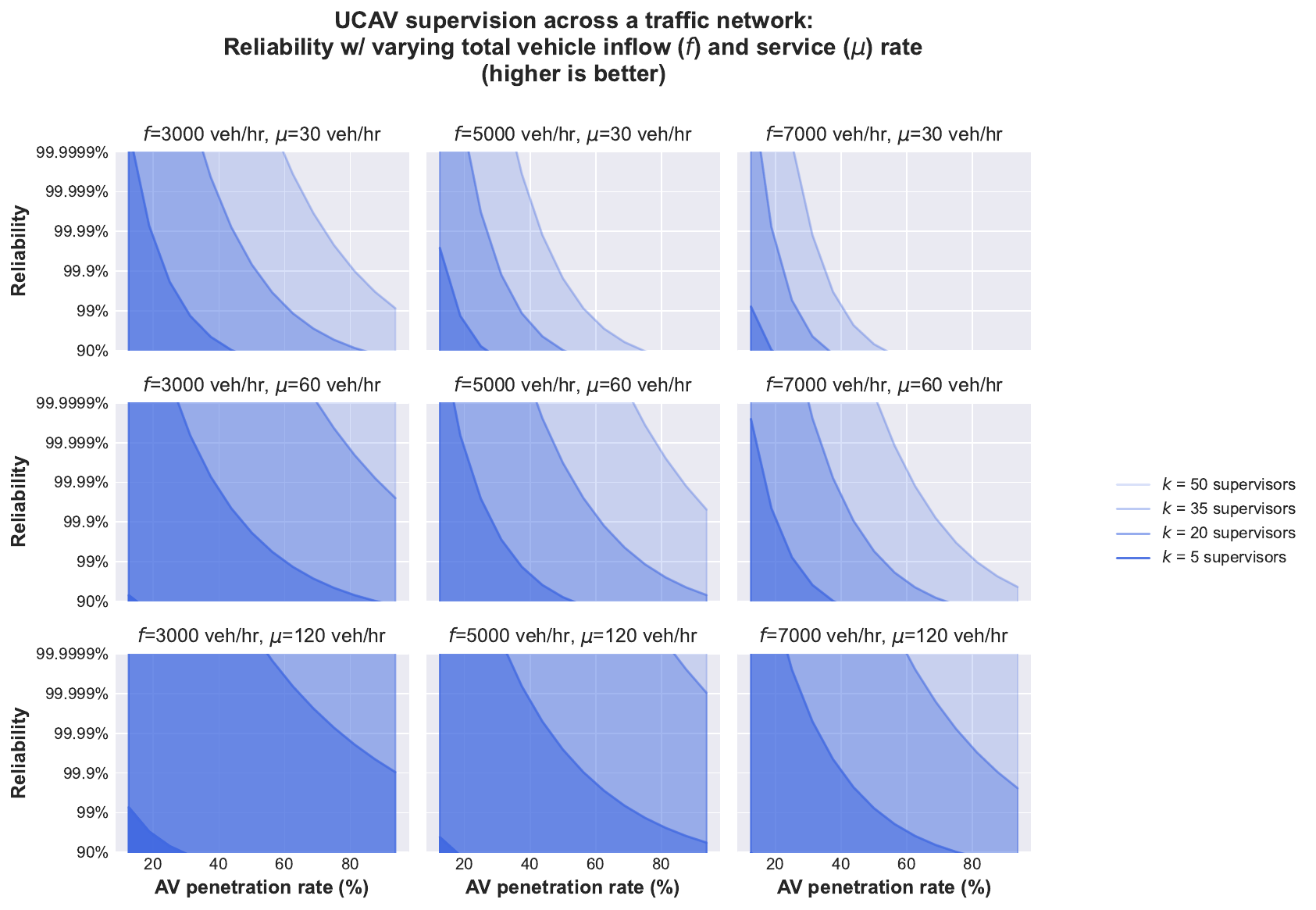}
    \caption{In the baseline case, as AV penetration increases, the system's ability to maintain adequate supervision drops. Note the logarithmic y-axis. For these plots $P(C_{\text{in}}) = 40\%$; the trends hold more generally.}
    \label{fig:geneqn_confidence_array}
\end{figure*}

\subsection{Tables with numerical data for examples in~\Cref{subsec:in-ring-reachability-induced-case}}
\label{appsec:tables}
See~\Cref{tab:numerical_probabilities21} through~\Cref{tab:numerical_probabilities32}.

\begin{table}[thbp]
    \centering
    \caption{Data for the case described in~\Cref{sec:sa_case2}. Relative reduction in supervision time for $S$ CCAVs relative to connected AV baseline $P(H_i \leq d_i) = 0.1$ for non-uniform distribution when $d_i=0.1$ and the total number of in-ring vehicles is 16. The column second from the left is an indicator of the distribution shift caused by the CCAVs. The column second from the right is the overall absolute supervision probability that the CCAVs typically achieve, and is thus a function of both gains via the implicit distribution shift indicated in the prior column and gains via the cooperative blocking behavior itself. The rightmost column indicates the relative \% improvement over the upper bound $d_i = 0.1$.}  
    \label{tab:numerical_probabilities21}
    \begin{tabular}{|l|c|c|p{0.16\linewidth}|}
    \hline
    \rule{0pt}{2ex} $S$ & $P(H_i \leq d_i)$ & $P(H_i \leq d_i, H_i \leq A_{(1)})$ & Relative improvement\\
    \hline
    2 & 0.0914 & 0.0749 & 25.12\%\\
    3 & 0.0894 & 0.0675 & 32.46\%\\
    4 & 0.0888 & 0.0614 & 38.60\%\\
    5 & 0.0891 & 0.0564 & 43.64\%\\
    6 & 0.0902 & 0.0522 & 47.80\%\\
    7 & 0.0917 & 0.0487 & 51.29\%\\
    8 & 0.0934 & 0.0457 & 54.26\%\\
    9 & 0.0952 & 0.0432 & 56.85\%\\
    10 & 0.0970 & 0.0409 & 58.13\%\\
    11 & 0.0989 & 0.0388 & 61.18\%\\
    12 & 0.1007 & 0.0370 & 63.03\%\\
    13 & 0.1025 & 0.0353 & 64.72\%\\
    14 & 0.1042 & 0.0337 & 66.27\%\\
    15 & 0.1058 & 0.0323 & 67.71\%\\
    16 & 0.1074 & 0.0310 & 69.03\%\\
    \hline
    \end{tabular}
\end{table}

\begin{table}[thbp]
    \centering
    \caption{Data for the case described in~\Cref{sec:sa_case2} when $d_i=0.01$.}
    \label{tab:numerical_probabilities22}
    \begin{tabular}{|l|c|c|p{0.16\linewidth}|}
    \hline
    \rule{0pt}{2ex} $S$ & $P(H_i \leq d_i)$ & $P(H_i \leq d_i, H_i \leq A_{(1)})$ & Relative improvement\\
    \hline
    2 & 0.0086 & 0.0084 & 16.00\%\\
    3 & 0.0081 & 0.0078 & 21.86\%\\
    4 & 0.0077 & 0.0074 & 26.43\%\\
    5 & 0.0074 & 0.0070 & 29.90\%\\
    6 & 0.0072 & 0.0067 & 32.52\%\\
    7 & 0.0071 & 0.0065 & 34.53\%\\
    8 & 0.0070 & 0.0064 & 36.12\%\\
    9 & 0.0069 & 0.0063 & 37.40\%\\
    10 & 0.0069 & 0.0062 & 38.47\%\\
    11 & 0.0069 & 0.0061 & 39.38\%\\
    12 & 0.0069 & 0.0060 & 40.17\%\\
    13 & 0.0069 & 0.0059 & 40.88\%\\
    14 & 0.0069 & 0.0058 & 41.51\%\\
    15 & 0.0069 & 0.0058 & 42.10\%\\
    16 & 0.0069 & 0.0057 & 42.63\%\\
    \hline
    \end{tabular}
\end{table}

\begin{table}[thbp]
    \centering
    \caption{Data for the case described in~\Cref{sec:sa_case3} when $d_i=0.1$}
    \label{tab:numerical_probabilities31}
    \begin{tabular}{|l|c|c|p{0.16\linewidth}|}
    \hline
    \rule{0pt}{2ex} $S$ & $P(H_i \leq d_i)$ & $P(H_i \leq d_i, H_i \leq A_{(1)})$ & Relative improvement\\
    \hline
    2  & 0.0805192 & 0.0639507 & 36.05\%\\
    3  & 0.0813434 & 0.0594455 & 40.55\%\\
    4  & 0.0828135 & 0.0554341 & 44.57\%\\
    5  & 0.084711  & 0.0519237 & 48.08\%\\
    6  & 0.0868491 & 0.0488549 & 51.14\%\\
    7  & 0.0890942 & 0.0461509 & 53.85\%\\
    8  & 0.0913594 & 0.0437418 & 56.26\%\\
    9  & 0.0935922 & 0.0415718 & 58.43\%\\
    10 & 0.0957621 & 0.0395984 & 60.40\%\\
    11 & 0.0978524 & 0.0377901 & 62.21\%\\
    12 & 0.0998548 & 0.036123  & 63.88\%\\
    13 & 0.101766  & 0.0345786 & 65.42\%\\
    14 & 0.103586  & 0.0331427 & 66.86\%\\
    15 & 0.105317  & 0.0318034 & 68.20\%\\
    16 & 0.106961  & 0.0305513 & 69.45\%\\
    \hline
    \end{tabular}
\end{table}

\begin{table}[thbp]
    \centering
    \caption{Data for the case described in~\Cref{sec:sa_case3} when $d_i=0.01$.}
    \label{tab:numerical_probabilities32}
    \begin{tabular}{|l|c|c|p{0.16\linewidth}|}
    \hline
    \rule{0pt}{2ex} $S$ & $P(H_i \leq d_i)$ & $P(H_i \leq d_i, H_i \leq A_{(1)})$ & Relative improvement\\
    \hline
    2  & 0.00714992 & 0.00696878 & 30.31\% \\
    3  & 0.00694662 & 0.00670019 & 33.00\% \\
    4  & 0.00680191 & 0.00648492 & 35.15\% \\
    5  & 0.00670619 & 0.00631588 & 36.84\% \\
    6  & 0.00664762 & 0.00618278 & 38.17\% \\
    7  & 0.00661575 & 0.00607607 & 39.23\% \\
    8  & 0.00660267 & 0.00598828 & 40.12\% \\
    9  & 0.00660279 & 0.00591401 & 40.86\% \\
    10 & 0.00661227 & 0.00584953 & 41.50\% \\
    11 & 0.0066285  & 0.00579226 & 42.08\% \\
    12 & 0.0066497  & 0.00574043 & 42.60\% \\
    13 & 0.0066746  & 0.00569277 & 43.07\% \\
    14 & 0.00670233 & 0.00564842 & 43.52\% \\
    15 & 0.00673222 & 0.00560671 & 43.93\% \\
    16 & 0.00676381 & 0.00556716 & 44.33\% \\
    \hline
    \end{tabular}
\end{table}

\subsection{Supervision requirements across a traffic network}
\label{subsec:traffic_network_arrays}

\textcolor{black}{
See Figures~\ref{fig:network_req_sup_array}, \ref{fig:network_num_avs_array}, \ref{fig:network_conf_noncoop_array} and~\ref{fig:network_conf_coop_array}. These illustrate similar trends for a variety of configurations.}

\begin{figure*}[tbhp]
    \centering
    \includegraphics[width=0.85\textwidth]{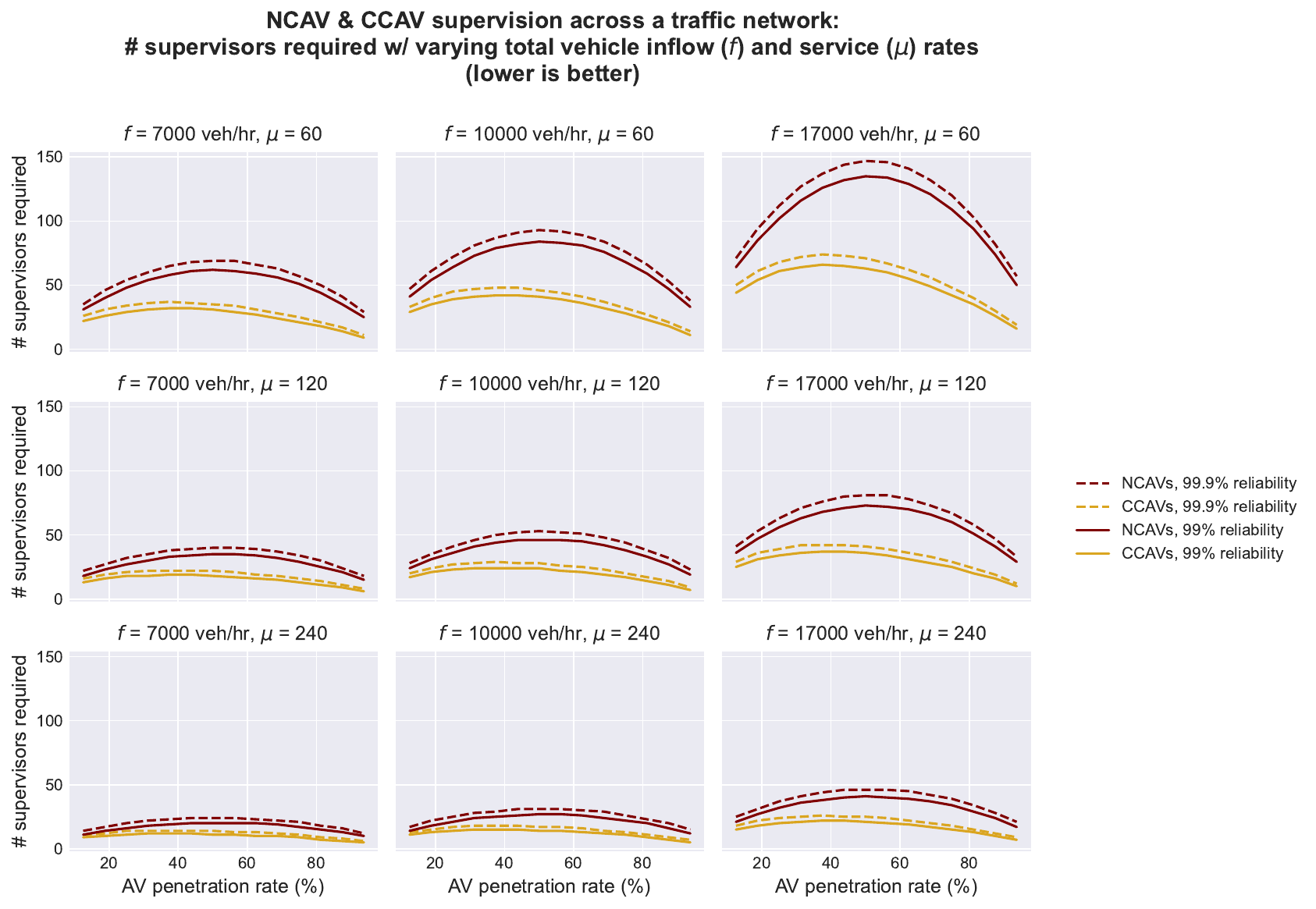}
    \caption{\textcolor{black}{Results described in~\Cref{fig:network_req_sup_and_avs_single} for a variety of total vehicle inflow rates $f$ and service rates $\mu$. The trends are similar across configurations.}}
    \label{fig:network_req_sup_array}
\end{figure*}

\begin{figure*}[tbhp]
    \centering
    \includegraphics[width=0.85\textwidth]{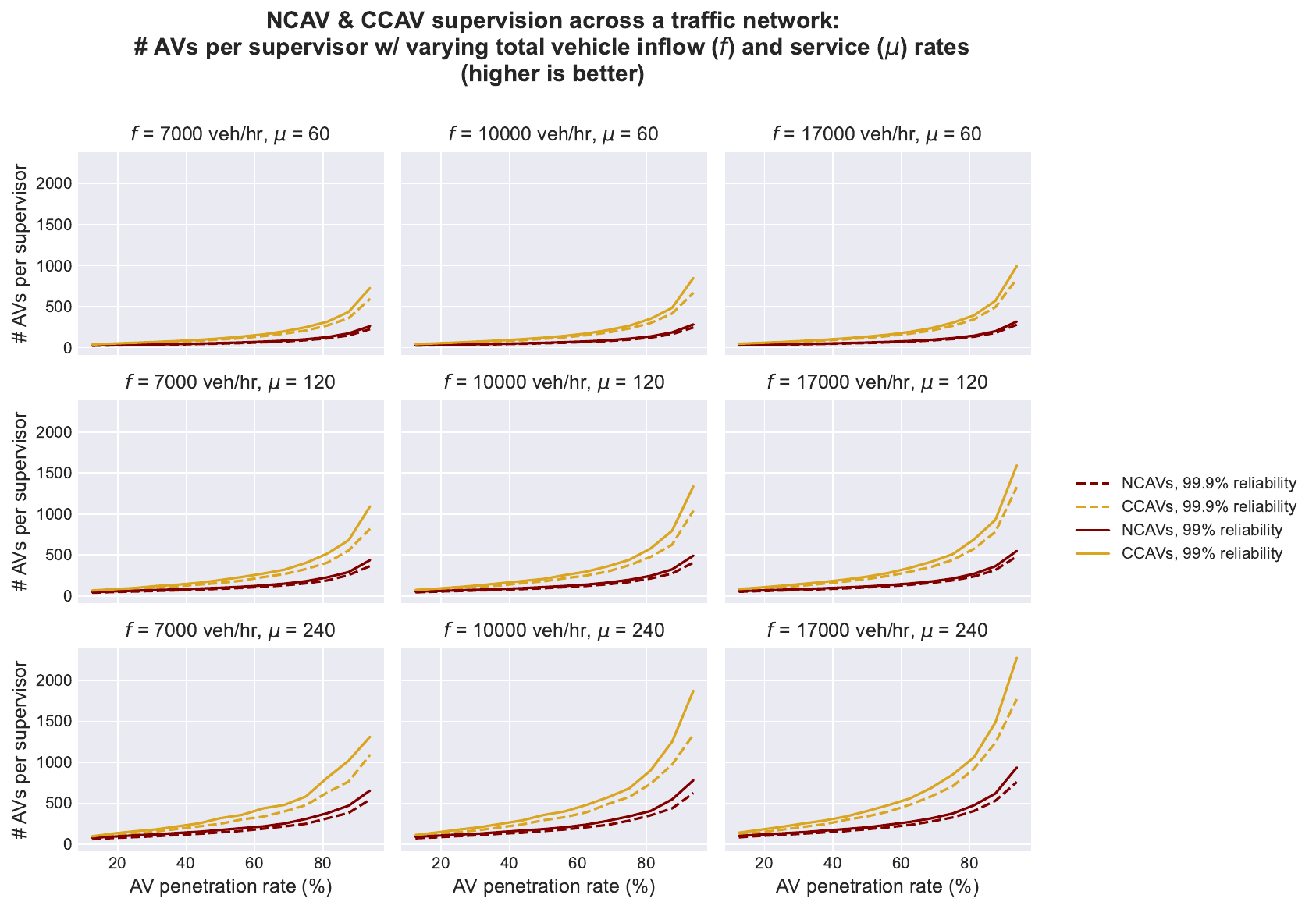}
    \caption{Counterintuitively, due to the benefits of CCAVs, supervisors can monitor substantially more AVs as they become more prevalent in the system.}
    \label{fig:network_num_avs_array}
\end{figure*}

\begin{figure*}[tbhp]
    \centering
    \includegraphics[width=0.85\textwidth]{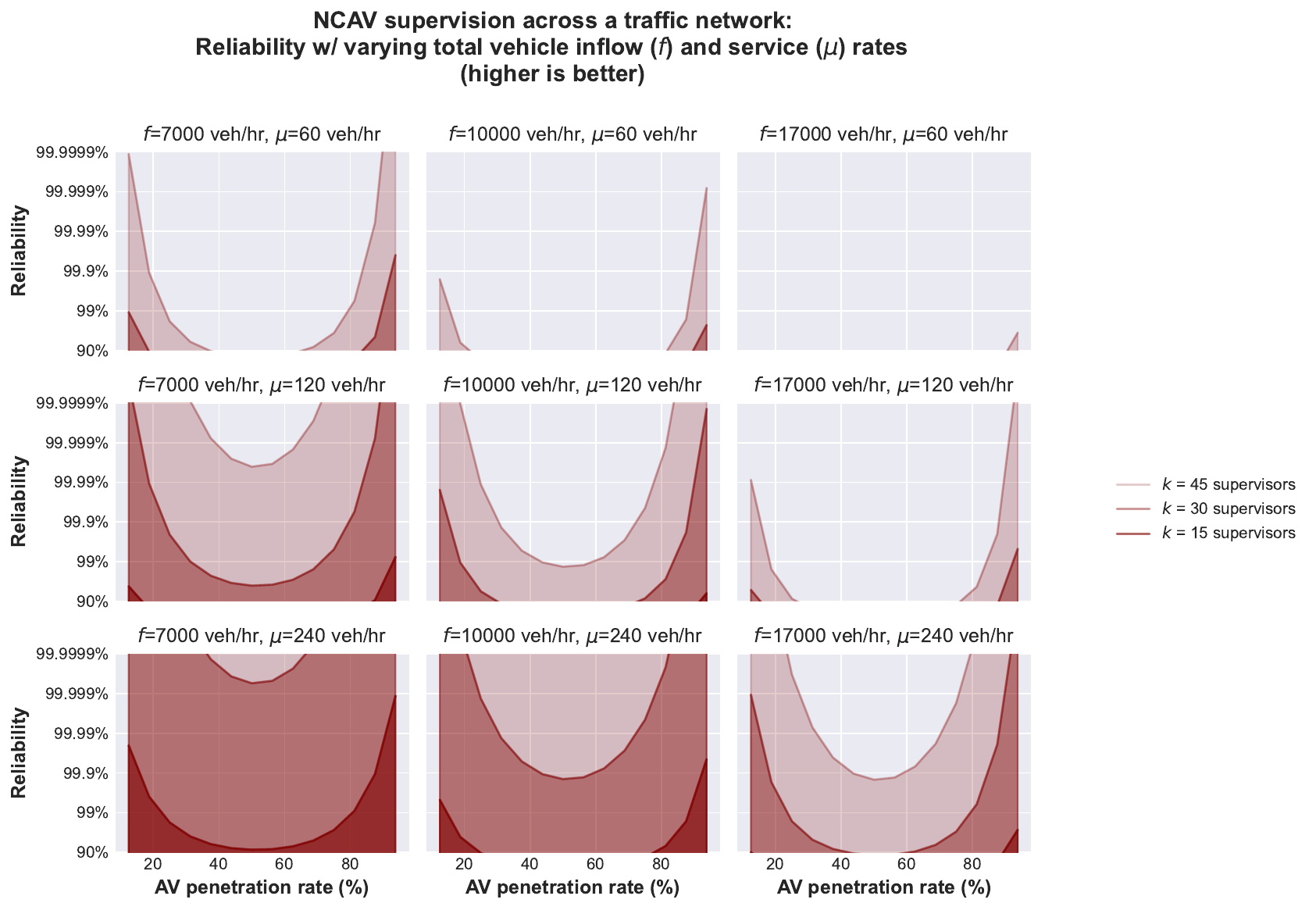}
    \caption{\textcolor{black}{Reliability} thresholds NCAVs can achieve at various task arrival and service rates. Trends are similar across configurations.}
    \label{fig:network_conf_noncoop_array}
\end{figure*}

\begin{figure*}[tbhp]
    \centering
    \includegraphics[width=0.85\textwidth]{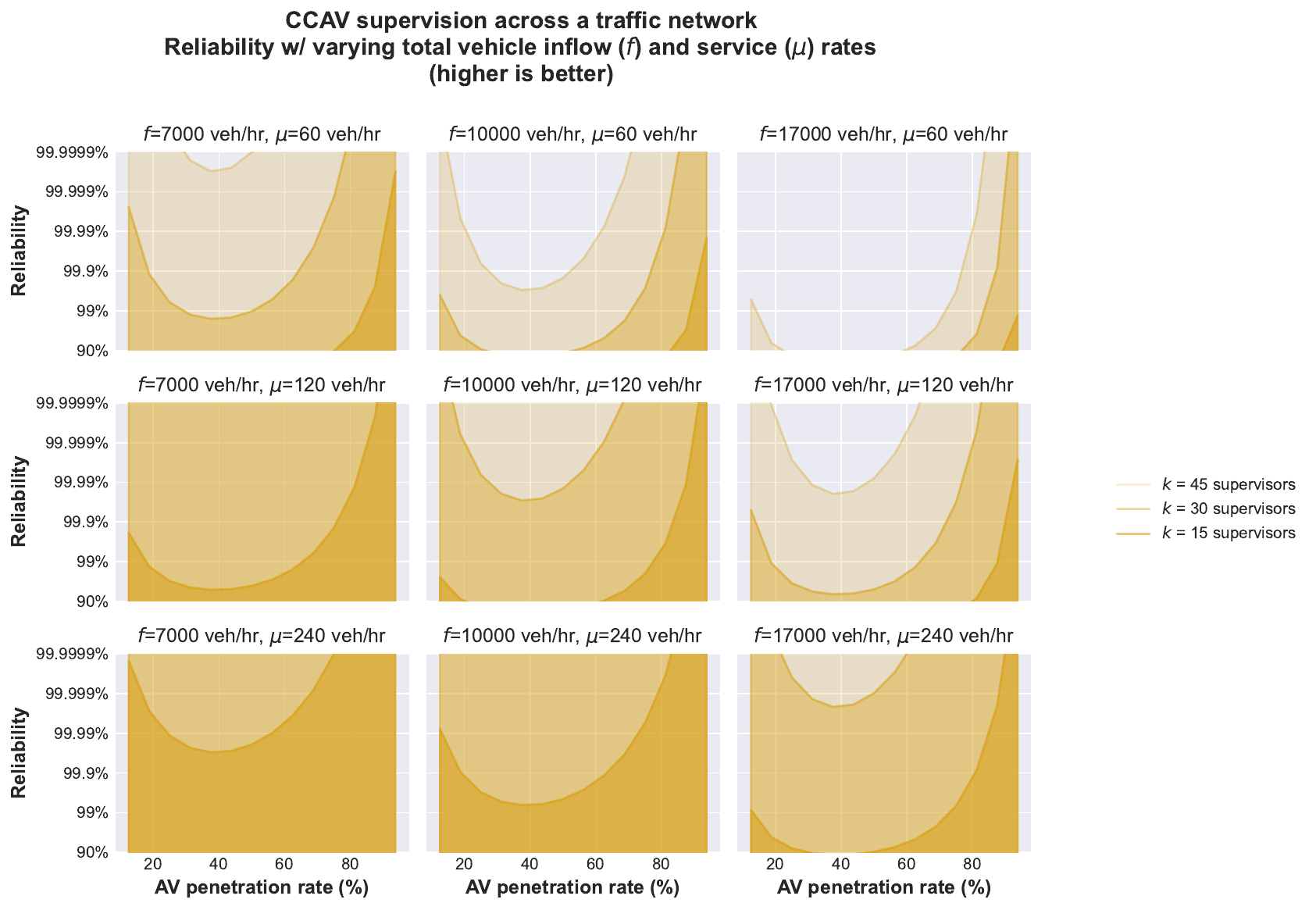}
    \caption{Comparing this figure with the one above---while making special note of the logarithmic y-axis---illustrates how CCAVs enable a fixed number of supervisors to achieve a substantially higher \textcolor{black}{reliability} threshold.}
    \label{fig:network_conf_coop_array}
\end{figure*}

\subsection{Supervision requirements for a single rotary}
See~\Cref{fig:rotary_conf_both}. 
\textcolor{black}{The single rotary trends mirror those from the much larger traffic network.}



\begin{figure*}[tbhp]
    \centering
    \includegraphics[width=0.85\textwidth]{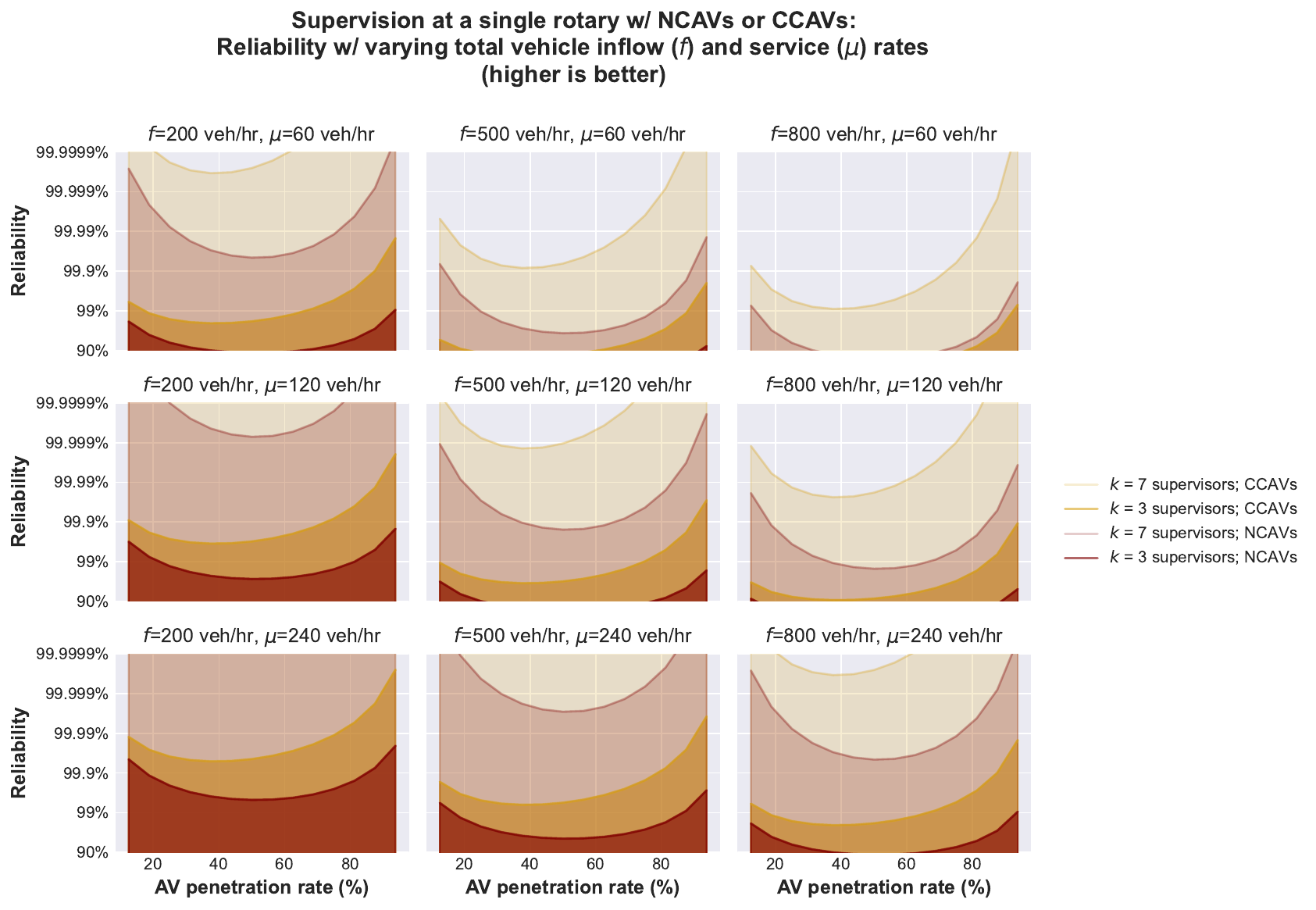}
    \caption{The \textcolor{black}{reliability} behavior is the same for isolated rotaries as for the larger system.}
    \label{fig:rotary_conf_both}
\end{figure*}



}


 

\begin{thebibliography}{1}
\bibliographystyle{IEEEtran}


\bibitem{wu2018learning} Wu, C. Learning and Optimization for Mixed Autonomy Systems-A Mobility Context. (UC Berkeley, 2018)

\bibitem{Allianz2018} Allianz Partners, Allianz Partners Self-driving Cars Hit a Speedbump; Interest in Autonomous Vehicle Technology Slows Down, (2018), [Online]. https://www.prnewswire.com/news-releases/self-driving-cars-hit-a-speedbump-interest-in-autonomous-vehicle-technology-slows-down-300714345.html




\bibitem{bila2016vehicles} Bila, C., Sivrikaya, F., Khan, M., \& Albayrak, S. Vehicles of the future: A survey of research on safety issues. {\em IEEE Transactions on Intelligent Transportation Systems}. \textbf{18}, 1046-1065 (2016)

\bibitem{nascimento2019systematic} Nascimento, A., Vismari, L., Molina, C., Cugnasca, P., Camargo, J., de Almeida, J., Inam, R., Fersman, E., Marquezini, M., \& Hata, A. A systematic literature review about the impact of artificial intelligence on autonomous vehicle safety. {\em IEEE Transactions on Intelligent Transportation Systems}. \textbf{21}, 4928-4946 (2019)

\bibitem{arbabzadeh2017data} Arbabzadeh, N. \& Jafari, M. A data-driven approach for driving safety risk prediction using driver behavior and roadway information data. {\em IEEE Transactions on Intelligent Transportation Systems}. \textbf{19}, 446-460 (2017)

\bibitem{vanholme2012highly} Vanholme, B., Gruyer, D., Lusetti, B., Glaser, S., \& Mammar, S. Highly automated driving on highways based on legal safety. {\em IEEE Transactions on Intelligent Transportation Systems}. \textbf{14}, 333-347 (2012)

\bibitem{saunders2017trial} Saunders, W., Sastry, G., Stuhlmueller, A. \& Evans, O. Trial without error: Towards safe reinforcement learning via human intervention. {\em Proc. 17th Int. Conf. Auton. Agents and Multiagent Syst.} pp. 2067-2069 (2018)

\bibitem{chen2014human} Chen, J.Y.C. \& Barnes, M.J. Human--agent teaming for multirobot control: A review of human factors issues. {\em IEEE Transactions on Human-Machine Systems} \textbf{44} pp. 13-29 (2014)

\bibitem{ziebinski2017review} Ziebinski, A., Cupek, R., Grzechca, D. \& Chruszczyk, L. Review of advanced driver assistance systems (ADAS). {\em AIP Conference Proceedings} \textbf{1906} (2017)

\bibitem{FHA2019} U.S. DoT Federal Highway Administration Highway Statistics 2021: Licensed Drivers. (U.S. Bureau of Transportation Statistics, 2021), https://www.fhwa.dot.gov/policyinformation/statistics/2021/dv1c.cfm

\bibitem{hopkin1988air} Hopkin, V.D. Air traffic control (Ch. 19 in \textit{Human factors in aviation}, edited by E.L. Winer and D.C. Nagel). {\em Academic Press} pp. 639-663 (1988)


\bibitem{wu2021flow} Wu, C., Kreidieh, A. R., Parvate, K., Vinitsky, E., \& Bayen, A. M. Flow: A Modular Learning Framework for Mixed Autonomy Traffic. IEEE Transactions on Robotics. (2021)


\bibitem{barth2008real} Barth, M. \& Boriboonsomsin, K. Real-world carbon dioxide impacts of traffic congestion. {\em Transportation Research Record}. \textbf{2058}, 163-171 (2008)

\bibitem{jula2000collision} Jula, H., Kosmatopoulos, E. \& Ioannou, P. Collision avoidance analysis for lane changing and merging. {\em IEEE Transactions On Vehicular Technology}. \textbf{49}, 2295-2308 (2000)

\bibitem{zhou2016impact} Zhou, M., Qu, X. \& Jin, S. On the impact of cooperative autonomous vehicles in improving freeway merging: a modified intelligent driver model-based approach. {\em IEEE Transactions On Intelligent Transportation Systems}. \textbf{18}, 1422-1428 (2016)

\bibitem{rios2016automated} Rios-Torres, J. \& Malikopoulos, A. Automated and cooperative vehicle merging at highway on-ramps. {\em IEEE Transactions On Intelligent Transportation Systems}. \textbf{18}, 780-789 (2016)

\bibitem{chen2020hierarchical} Chen, N., van Arem, B., Alkim, T., \& Wang, M. A hierarchical model-based optimization control approach for cooperative merging by connected automated vehicles. {\em IEEE Transactions on Intelligent Transportation Systems}. (2020)

\bibitem{rios2016survey} Rios-Torres, J. \& Malikopoulos, A. A survey on the coordination of connected and automated vehicles at intersections and merging at highway on-ramps. {\em IEEE Transactions on Intelligent Transportation Systems}. \textbf{49}, 1066-1077 (2016)

\bibitem{zheng2019cooperative} Zheng, Y., Ran, B., Qu, X., Zhang, J., \& Lin, Y. Cooperative lane changing strategies to improve traffic operation and safety nearby freeway off-ramps in a connected and automated vehicles environment. {\em IEEE Transactions on Intelligent Transportation Systems}. \textbf{21}, 4605-4614 (2019)

\bibitem{yu2019impact} Yu, H., Tak S., Park, M., \& Yeo, H. Impact of autonomous-vehicle-only lanes in mixed traffic conditions. {\em Transportation research record}. \textbf{2673}, 430-439 (2019)

\bibitem{erdogan2008geographical} Erdogan, S., Yilmaz, I., Baybura, T., \& Gullu, M. Geographical information systems aided traffic accident analysis system case study: city of Afyonkarahisar. {\em Accident Analysis \& Prevention}, 40(1), pp. 174-181 (2008)

\bibitem{hoque2022fleet} Hoque, R., Chen, L.Y., Sharma, S., Dharmarajan, K., Thananjeyan, B., Abbeel, P., \& Goldberg, K. Fleet-DAgger: Interactive Robot Fleet Learning with Scalable Human Supervision. {\em Proc. 6th Conf. Robot Learning}, pp. 368-380 (2023)

\bibitem{sheridan1986human} Sheridan, T. Human supervisory control of robot systems. {\em Proceedings. 1986 IEEE International Conference on Robotics and Automation} \textbf{3} pp. 808-812 (1986)

\bibitem{kirwan2001role} Kirwan, B. The role of the controller in the accelerating industry of air traffic management. {\em Safety Science} \textbf{37} pp. 151-185 (2001)

\bibitem{van2020meaningful} van den Broek, J., Griffioen, J.R. \& van der Drift, M. Meaningful Human Control in Autonomous Shipping: An Overview. {\em IOP Conference Series: Materials Science and Engineering}. \textbf{929} (2020)

\bibitem{bellemare2020autonomous} Bellemare, M., Candido, S., Castro, P., Gong, J., Machado, M., Moitra, S., Ponda, S. \& Wang, Z. Autonomous navigation of stratospheric balloons using reinforcement learning. {\em Nature}. \textbf{588}, 77-82 (2020)

\bibitem{hambling2020ai} Hambling, D. AI outguns a human fighter pilot. (Elsevier, 2020)


\bibitem{silver2017mastering} Silver, D., Schrittwieser, J., Simonyan, K., Antonoglou, I., Huang, A., Guez, A., Hubert, T., Baker, L., Lai, M., Bolton, A. \& Others Mastering the game of go without human knowledge. {\em Nature}. \textbf{550}, 354-359 (2017)

\bibitem{bacci2020probabilistic} Bacci, E. \& Parker, D. Probabilistic guarantees for safe deep reinforcement learning. {\em International Conference On Formal Modeling And Analysis Of Timed Systems}. pp. 231-248 (2020)

\bibitem{bouton2019reinforcement} Bouton, M., Karlsson, J., Nakhaei, A., Fujimura, K., Kochenderfer, M.J. \& Tumova, J. Reinforcement learning with probabilistic guarantees for autonomous driving. {\em ArXiv Preprint ArXiv:1904.07189}. (2019)

\bibitem{le2022survey} Le Mero, L., Yi, D., and Dianati, M. \& Mouzakitis, A. A survey on imitation learning techniques for end-to-end autonomous vehicles. {\em IEEE Transactions on Intelligent Transportation Systems}. (2022)

\bibitem{kelly2019hg} Kelly, M., Sidrane, C., Driggs-Campbell, K. \& Kochenderfer, M.J. Hg-dagger: Interactive imitation learning with human experts. {\em 2019 International Conference on Robotics and Automation (ICRA)}. pp. 8077-8083 (2019)

\bibitem{drew2021multi} Drew, D.S. Multi-agent systems for search and rescue applications. {\em Current Robotics Reports}. pp. 1-12 (2021)

\bibitem{dahiya2021scalable} Dahiya, A., Akbarzadeh, N., Mahajan, A., \& Smith, S.L. Scalable operator allocation for multi-robot assistance: a restless bandit approach. {\em IEEE Transactions on Control of Network Systems}. pp. 1 (2021)

\bibitem{cummings2007automation} Cummings, M., Bruni, S., Mercier, S. \& Mitchell, P. Automation architecture for single operator, multiple UAV command and control. (Massachusetts Inst Of Tech Cambridge, 2007)

\bibitem{cummings2007predicting} Cummings, M., Nehme, C., Crandall, J. \& Mitchell, P. Predicting operator capacity for supervisory control of multiple UAVs. {\em Innovations In Intelligent Machines-1}. pp. 11-37 (2007)

\bibitem{kidwell2012adaptable} Kidwell, B., Calhoun, G., Ruff, H. \& Parasuraman, R. Adaptable and adaptive automation for supervisory control of multiple autonomous vehicles. {\em Proceedings Of The Human Factors And Ergonomics Society Annual Meeting}. \textbf{56}, 428-432 (2012)

\bibitem{humann2019human} Humann, J. \& Pollard, K. Human factors in the scalability of multirobot operation: A review and simulation. {\em 2019 IEEE International Conference On Systems, Man And Cybernetics (SMC)}. pp. 700-707 (2019)

\bibitem{conesa2015distributed} Conesa-Muñoz, J., Soto, M., Santos, P. \& Ribeiro, A. Distributed multi-level supervision to effectively monitor the operations of a fleet of autonomous vehicles in agricultural tasks. {\em Sensors}. \textbf{15}, 5402-5428 (2015)

\bibitem{olsen2003metrics} Olsen, D. \& Goodrich, M. Metrics for evaluating human-robot interactions. {\em Proceedings Of PERMIS}. \textbf{2003} pp. 4 (2003)

\bibitem{olsen2004fan} Olsen Jr, D. \& Wood, S. Fan-out: Measuring human control of multiple robots. {\em Proceedings Of The SIGCHI Conference On Human Factors In Computing Systems}. pp. 231-238 (2004)

\bibitem{hampshire2020beyond} Hampshire, R.C., Bao, S., Lasecki, W.S., Daw, A. \& Pender, J. Beyond safety drivers: Applying air traffic control principles to support the deployment of driverless vehicles. {\em PLoS one}. \textbf{15} (2020)

\bibitem{daw2019beyond} Daw, A., Hampshire, R.C. \& Pender, J. Beyond safety drivers: staffing a teleoperations system for autonomous vehicles {\em ArXiv preprint ArXiv:1907.12650}. (2019)

\bibitem{swamy2020scaled} Swamy, G., Reddy, S., Levine, S. \& Dragan, A. Scaled autonomy: enabling human operators to control robot fleets. {\em 2020 IEEE International Conference On Robotics And Automation (ICRA)}. pp. 5942-5948 (2020)

\bibitem{ahn2020robust} Ahn, H., Rizzi, A., \& Colombo, A. Robust supervisors for intersection collision avoidance in the presence of legacy vehicles. {\em International Journal of Control, Automation and Systems}. \textbf{18}, 384-393 (2020)

\bibitem{altche2017algorithm} Altch{\'e}, F., Qian, X., \& de la Fortelle, A. A survey on the coordination of connected and automated vehicles at intersections and merging at highway on-ramps. {\em IEEE Transactions on Intelligent Transportation Systems}. \textbf{18}, 3527-3539 (2017)








\bibitem{althoff2020set} Althoff, M., Frehse, G. \& Girard, A. Set Propagation Techniques for Reachability Analysis. {\em Annual Review Of Control, Robotics, And Autonomous Systems}. \textbf{4} (2020)

\bibitem{kurzhanski2000ellipsoidal} Kurzhanski, A. \& Varaiya, P. Ellipsoidal techniques for reachability analysis. {\em International Workshop On Hybrid Systems: Computation And Control}. pp. 202-214 (2000)

\bibitem{liebenwein2018sampling} Liebenwein, L., Baykal, C., Gilitschenski, I., Karaman, S. \& Rus, D. Sampling-based approximation algorithms for reachability analysis with provable guarantees. {\em RSS}. (2018)

\bibitem{bansal2017hamilton} Bansal, S., Chen, M., Herbert, S. \& Tomlin, C. Hamilton-Jacobi reachability: A brief overview and recent advances. {\em 2017 IEEE 56th Annual Conference On Decision And Control (CDC)}. pp. 2242-2253 (2017)

\bibitem{bahati2020multi} Bahati, G., Gibson, M. \& Bayen, A. Multi-Adversarial Safety Analysis for Autonomous Vehicles. {\em RSS Workshop on Robust Autonomy}. (2020)

\bibitem{leung2020infusing} Leung, K., Schmerling, E., Zhang, M., Chen, M., Talbot, J., Gerdes, J. \& Pavone, M. On infusing reachability-based safety assurance within planning frameworks for human–robot vehicle interactions. {\em The International Journal Of Robotics Research}. \textbf{39}, 1326-1345 (2020)






\bibitem{johns2016exploring} Johns, M., Mok, B., Sirkin, D., Gowda, N., Smith, C., Talamonti, W. \& Ju, W. Exploring shared control in automated driving. {\em 2016 11th ACM/IEEE International Conference On Human-Robot Interaction (HRI)}. pp. 91-98 (2016)

\bibitem{sugiyama2008traffic} Sugiyama, Y., Fukui, M., Kikuchi, M., Hasebe, K., Nakayama, A., Nishinari, K., Tadaki, S. \& Yukawa, S. Traffic jams without bottlenecks—experimental evidence for the physical mechanism of the formation of a jam. {\em New Journal Of Physics}. \textbf{10}, 033001 (2008)

\bibitem{treiber2000congested} Treiber, M., Hennecke, A. \& Helbing, D. Congested traffic states in empirical observations and microscopic simulations. {\em Physical Review E}. \textbf{62}, 1805 (2000)

\bibitem{treiber2013traffic} Treiber, M. \& Kesting, A. Traffic flow dynamics. {\em Traffic Flow Dynamics: Data, Models And Simulation, Springer-Verlag Berlin Heidelberg}. (2013)

\bibitem{rakha2004vehicle} Rakha, H., Snare, M. \& Dion, F. Vehicle dynamics model for estimating maximum light-duty vehicle acceleration levels {\em Transportation Research Record}. \textbf{1883} pp. 40-49 (2004)

\bibitem{brown2020taxonomy} Brown, K., Driggs-Campbell, K., \& Kochenderfer, M.J. A taxonomy and review of algorithms for modeling and predicting human driver behavior {\em ArXiv Preprint ArXiv:2006.08832}. (2020)

\bibitem{rengaraju1995vehicle} Rengaraju, V.R. \& Rao, V.T. Vehicle-arrival characteristics at urban uncontrolled intersections. {\em Journal of transportation engineering}. \textbf{121}, pp. 317-323 (1995)

\bibitem{sarla2020performance} Sarla, P., Doodipala, M.R. \& Endla, P. Performance of Vehicle Arrival Traffic Data at Fuel Station using Queuing System. {\em IOP Conference Series: Materials Science and Engineering}. \textbf{981}, (2020)

\bibitem{mirchandani2006analytical} Mirchandani, P. \& Zou, N. Analytical modeling of ramp meter systems. {\em 2006 IEEE Intelligent Transportation Systems Conference}. pp. 432-436 (2006)

\bibitem{ding2020penetration} Ding, J., Peng, H., Zhang, Y. \& Li, L. Penetration effect of connected and automated vehicles on cooperative on-ramp merging {\em IET Intelligent Transport Systems}. \textbf{14}, pp. 56-64. (2020)

\bibitem{zhang2018analysis} Zhang, R., Rossi, F. \& Pavone, M. Analysis, control, and evaluation of mobility-on-demand systems: A queueing-theoretical approach. {\em IEEE Transactions on Control of Network Systems}. \textbf{6}, pp. 115-126 (2018)

\bibitem{gross2008fundamentals} Gross, D. Fundamentals of queueing theory. {\em John Wiley \& Sons}. (2008)

\bibitem{rasekhipour2016potential} Rasekhipour, Y., Khajepour, A., Chen, S.K. \& Litkouhi, B. A potential field-based model predictive path-planning controller for autonomous road vehicles {\em IEEE Transactions on Intelligent Transportation Systems}. \textbf{18} pp. 1255-1267 (2016)

\bibitem{grimmett2020} Grimmett, G. R., and Stirzaker, D. R.. Probability and random processes. Oxford university press, 2020.





\bibitem{pilko2007safety} Pilko, P., Bared, J.G., Edara, P.K. \& Kim, T. Safety Assessment of Interchange Spacing on Urban Freeways. (FHWA-HRT-07-031, Office of Research, Development, and Technology, US Federal Highway Administration, 2007), https://www.fhwa.dot.gov/publications/research/safety/07031/07031.pdf

\bibitem{hess1963capacities} Hess, J.W. Capacities and Characteristics of Ramp-freeway Connections. {\em Highw. Res. Rec.} \textbf{3-27} pp. 69-115 (1963) https://onlinepubs.trb.org/Onlinepubs/hrr/1963/27/27-004.pdf

\bibitem{mcguckin2018summary} McGuckin, N. \& Fucci, A. Summary of travel trends: 2017 national household travel survey. (US Department of Transportation, Federal Highway Administration, 2018), https://nhts.ornl.gov/assets/2017\_nhts\_summary\_travel\_trends.pdf

\bibitem{FHA2018_VMT} Federal Highway Administration. Average Annual PMT, VMT Person Trips and Trip Length by Trip Purpose. (U.S. Bureau of Transportation Statistics, 2018), https://www.bts.gov/content/average-annual-pmt-vmt-person-trips-and-trip-length-trip-purpose

\bibitem{usdot_2015status} Federal Highway Administration (FHWA). ``2015 Status of the Nation's Highways, Bridges, and Transit Conditions \& Performance,'' in Report to Congress. (FHWA: Washington, DC, USA, 2017)

\end{thebibliography}



 
\vspace{11pt}

\begin{IEEEbiography}[{\includegraphics[width=1in,height=1.25in,clip,keepaspectratio]{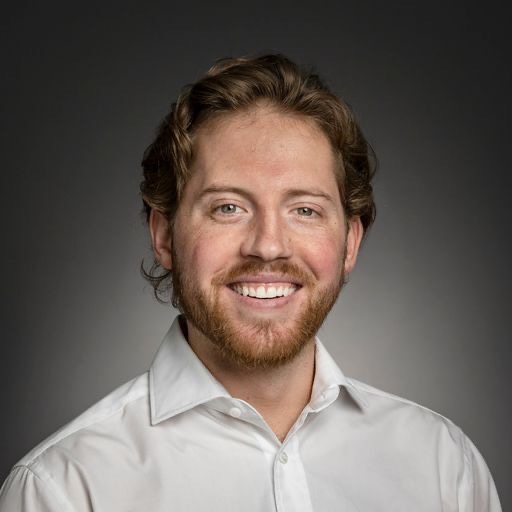}}]{Cameron Hickert}
earned a B.S. in physics from the University of Denver and an M.E. in computational science and engineering from Harvard University. He is currently a Ph.D. student in Social and Engineering Systems at MIT. His research interests include safe autonomy, reinforcement learning, and distributed cyber-physical systems.
\end{IEEEbiography}
\begin{IEEEbiography}[{\includegraphics[width=1in,height=1.25in,clip,keepaspectratio]{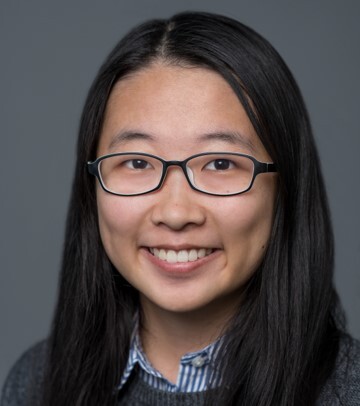}}]{Sirui Li}
received her B.S. degree with majors in computer science and mathematics from Washington University in St. Louis, MO, USA, in 2019. She is currently working toward her Ph.D. degree in Social and Engineering Systems at MIT, Cambridge, MA, USA. Her research interests include areas of machine learning for combinatorial optimization and control analysis for transportation systems.
\end{IEEEbiography}
\begin{IEEEbiography}[{\includegraphics[width=1in,height=1.25in,clip,keepaspectratio]{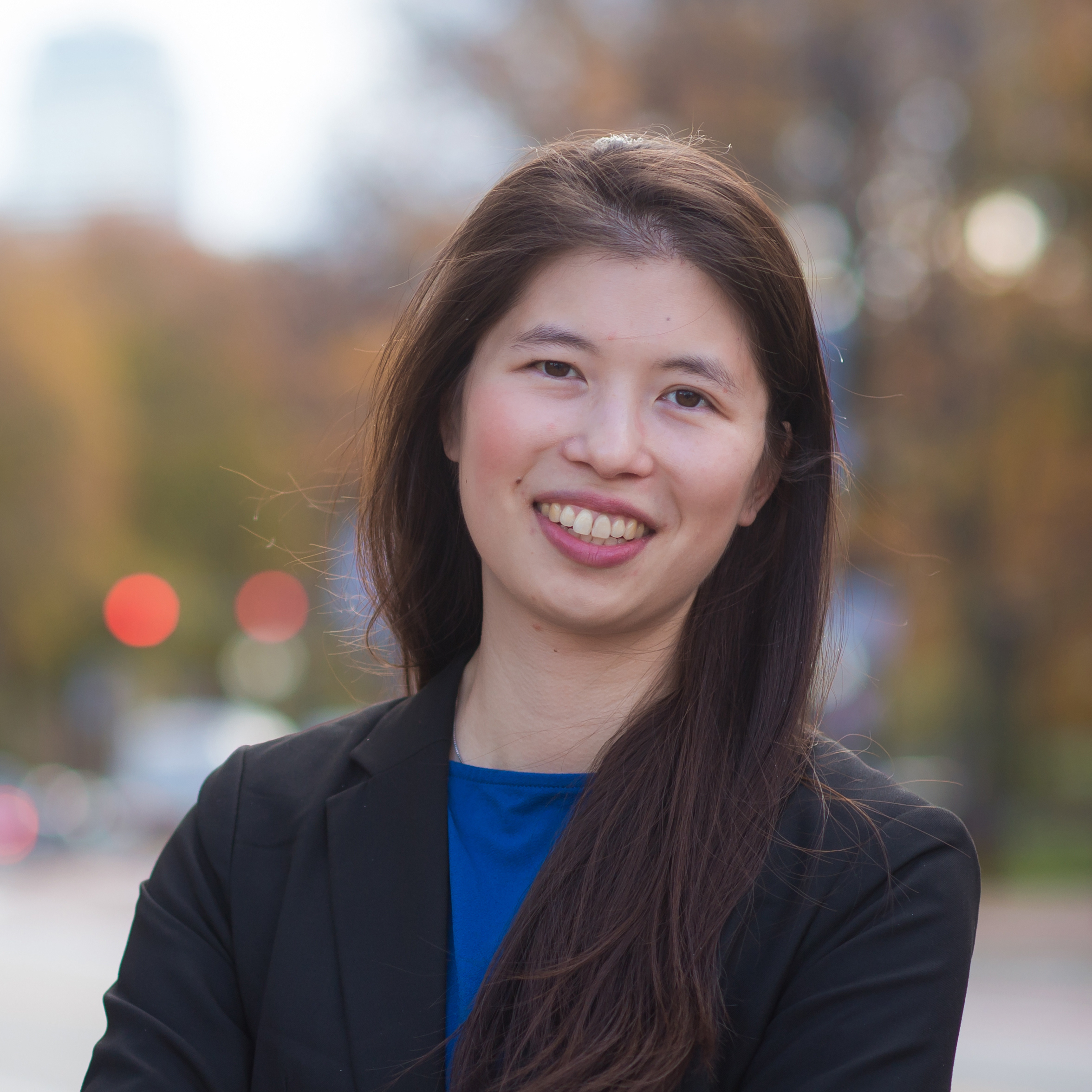}}]{Cathy Wu}
is an Assistant Professor at MIT in LIDS, CEE, and IDSS. She holds a Ph.D. from UC Berkeley and B.S. and M.Eng. from MIT, all in EECS, and completed a Postdoc at Microsoft Research. Her interests are at the intersection of machine learning, control, and mobility. Her recent research focuses on how learning-enabled methods can better cope with the complexity, diversity, and scale of control and operations in mobility systems. She is interested in developing principled computational tools to enable reliable decision-making in sociotechnical systems.
\end{IEEEbiography}



\vfill

\end{document}